\documentclass[twoside,11pt]{article}

\usepackage{amsmath,amsfonts,bm, amssymb}
\usepackage{amsthm}

\newcommand{\Set}{\mathcal{S}}

\newcommand{\defeq}{\mathrel{\mathop:}=}

\def\eqref#1{equation~\ref{#1}}

\def\1{\bm{1}}

\def\rmI{{\mathbf{I}}}

\def\rmQ{{\mathbf{Q}}}

\def\rmU{{\mathbf{U}}}

\def\rmX{{\mathbf{X}}}

\def\veta{{\bm{\eta}}}

\def\vd{{\bm{d}}}

\def\vo{{\bm{o}}}

\def\vt{{\bm{t}}}

\def\vv{{\bm{v}}}
\def\vw{{\bm{w}}}
\def\vx{{\bm{x}}}
\def\vy{{\bm{y}}}
\def\vz{{\bm{z}}}

\DeclareMathAlphabet{\mathsfit}{\encodingdefault}{\sfdefault}{m}{sl}
\SetMathAlphabet{\mathsfit}{bold}{\encodingdefault}{\sfdefault}{bx}{n}

\def\gA{{\mathcal{A}}}

\def\gD{{\mathcal{D}}}

\def\gF{{\mathcal{F}}}
\def\gG{{\mathcal{G}}}

\def\gL{{\mathcal{L}}}

\def\gN{{\mathcal{N}}}

\def\gR{{\mathcal{R}}}
\def\gS{{\mathcal{S}}}
\def\gT{{\mathcal{T}}}

\def\gW{{\mathcal{W}}}

\def\sP{{\mathbb{P}}}

\def\sR{{\mathbb{R}}}

\def\sZ{{\mathbb{Z}}}

\newcommand{\E}{\mathbb{E}}

\DeclareMathOperator*{\argmax}{arg\,max}
\DeclareMathOperator*{\argmin}{arg\,min}

\usepackage{jmlr2e}
\newtheorem{assumption}[theorem]{Assumption}

\usepackage{microtype}
\usepackage{graphicx}
\usepackage{subcaption}
\usepackage{booktabs}

\usepackage{hyperref}
\usepackage{float}

\usepackage{amsmath}
\usepackage{amssymb}
\usepackage{mathtools}
\usepackage{amsthm}
\usepackage{multirow}
\usepackage{multicol}
\usepackage{makecell}
\usepackage{xcolor}
\usepackage{arydshln}
\usepackage{algorithm}
\usepackage{algorithmic}

\usepackage[capitalize,noabbrev]{cleveref}

\theoremstyle{plain}
\usepackage[normalem]{ulem}

\newcommand{\change}[1]{{\color{black} #1}}
\newcommand{\edit}[1]{{\color{black} #1}}

\newcommand{\revision}[1]{{\color{black} #1}}

\usepackage{lastpage}
\jmlrheading{25}{2024}{1-\pageref{LastPage}}{8/22; Revised
5/23}{12/24}{22-0950}{Chen Liu, Zhichao Huang, Mathieu Salzmann, Tong Zhang, Sabine S\"usstrunk}
\ShortHeadings{title}{Liu, Huang, Salzmann, Zhang, S\"usstrunk}

\ShortHeadings{On the Impact of Hard Adversarial Examples on Overfitting in Adversarial Training}{Liu, Huang, Salzmann, Zhang and S\"usstrunk.}
\firstpageno{1}

\begin{document}

\title{On the Impact of Hard Adversarial Instances \\ on Overfitting in Adversarial Training}

\author{\name Chen Liu$^{1~\dagger~*}$ \email chen.liu@cityu.edu.hk
        \AND
        \name Zhichao Huang$^2$ \email zhichao.huang@connect.ust.hk
        \AND
        \name Mathieu Salzmann$^3$ \email mathieu.salzmann@epfl.ch
        \AND
        Tong Zhang$^{4~\ddagger}$ \email tongzhang@tongzhang-ml.org
        \AND
        \name Sabine S\"usstrunk$^3$ \email sabine.susstrunk@epfl.ch \\
        \AND
        \addr $^1$ Department of Computer Science, City University of Hong Kong \\
        83 Tai Chee Ave, Kowloon Tong, Hong Kong, China \\
        \addr $^2$ Department of Mathematics, Hong Kong University of Science and Technology \\
        Clear Water Bay, Hong Kong, China \\
        \addr $^3$ School of Computer and Communication Sciences, École Polytechnique Fédérale de Lausanne \\
        Rte Cantonale, 1015 Lausanne, Switzerland \\
        \addr $^4$ Siebel School of Computing and Data Science, University of Illinois Urbana-Champaign \\
        201 N Goodwin Ave, Urbana, IL 61801, USA \\
        \addr $\dagger$ Most of the work was done when Chen Liu was with École Polytechnique Fédérale de Lausanne. \\
        \addr $\ddagger$ The work was done when Tong Zhang was with Hong Kong University of Science and Technology. \\
        \addr $*$ Corresponding author
        }

\editor{Pradeep Ravikumar}

\setlength{\parindent}{0pt}
\setlength{\parskip}{0.5em}

\maketitle

\begin{abstract} %
Adversarial training is a popular method to robustify models against adversarial attacks.
However, it exhibits much more severe overfitting than training on clean inputs.
In this work, we investigate this phenomenon from the perspective of training instances, i.e., training input-target pairs.
\change{Based on a quantitative metric measuring the relative difficulty of an instance in the training set, we analyze the model's behavior on training instances of different difficulty levels}.
This lets us demonstrate that the decay in generalization performance of adversarial training is a result of fitting hard adversarial instances.
We theoretically verify our observations for both linear and general nonlinear models, proving that models trained on hard instances have worse generalization performance than ones trained on easy instances, and that this generalization gap \change{increases with the size of the adversarial budget}.
Finally, we investigate solutions to mitigate adversarial overfitting in several scenarios, including fast adversarial training and fine-tuning a pretrained model with additional data.
Our results demonstrate that using training data adaptively improves the model's robustness.
\end{abstract}

\begin{keywords}
  Robustness, overfitting, adversarial training, deep learning, optimization.
\end{keywords}

\section{Introduction} \label{sec:intro}

The existence of adversarial examples~\citep{szegedy2013intriguing} causes serious safety concerns when deploying modern deep learning models.
For example, for classification tasks, imperceptible perturbations of the input instance can fool state-of-the-art classifiers.
Many strategies to obtain models that are robust against adversarial attacks have been proposed~\citep{buckman2018thermometer, dhillon2018stochastic, ma2018characterizing, samangouei2018defense, pang2019improving, pang2020rethinking, xiao2020enhancing}, but most of them have been found to be ineffective in the presence of adaptive attacks~\citep{athalye2018obfuscated, croce2020reliable, tramer2020adaptive, croce2021mind}.
Ultimately, this leaves adversarial training~\citep{madry2017towards} and its variants~\citep{alayrac2019labels, carmon2019unlabeled, hendrycks2019using, sinha2019harnessing, zhang2019you, gowal2020uncovering, wu2020adversarial, 2021Improving, jiang2023towards, Wang2023BetterDM, Cui2024DecoupledKD, zhong2024efficient} as the most effective and popular approaches to construct robust models.
Unfortunately, adversarial training yields much worse performance on the test data than vanilla training.
In particular, it strongly suffers from overfitting~\citep{rice2020overfitting}, with the model's performance decaying significantly on the test set in the later phase of adversarial training.
\revision{Because modern deep neural networks have sufficient capacity to fit the training data perfectly, even under adversarial attacks, overfitting remains one of the primary challenges for improving model robustness on the test data.}
While the overfitting issue can be mitigated by early stopping~\citep{rice2020overfitting} or model smoothing~\citep{chen2021robust}, the reason behind the overfitting of adversarial training remains poorly understood.

In this paper, we study this phenomenon from the perspective of training instances, i.e., training input-target pairs.
We first introduce a quantitative metric, \revision{based on the percentile of the instance's loss objective}, to measure the relative difficulty of an instance within a training set.
Then, we analyze the model's behavior, such as its loss and intermediate activations, on training instances of different difficulty levels.
This lets us discover that the model's generalization performance decays significantly when it fits the hard adversarial instances in the later training phase.

To more rigorously study this phenomenon, we conduct theoretical analyses on both linear and nonlinear models.
For linear models, we study logistic regression on a Gaussian mixture model, in which we can calculate the analytical expression of the model parameters upon convergence and thus the robust test accuracy.
Our theorem demonstrates that adversarial training on harder instances leads to larger generalization gaps.
Furthermore, the difference in robust accuracy between the models trained by the hard instances and the ones trained by the easy instances increases with the size of the adversarial budget.
In the case of nonlinear models, we derive the lower bound of the model's Lipschitz constant when the model is well fit to the training instances under adversarial attacks.
This bound increases with the difficulty level of the training instances and the size of the adversarial budget.
Since a larger Lipschitz constant indicates a higher adversarial vulnerability~\citep{ruan2018reachability, weng2018towards, weng2018evaluating}, our theoretical analysis confirms our empirical observations.

Our empirical and theoretical analyses indicate that avoiding fitting the hard training instances can mitigate adversarial overfitting.
We therefore study this in three different scenarios: standard adversarial training, fast adversarial training and adversarial fine-tuning with additional training data.
We show that existing approaches that successfully mitigate adversarial overfitting~\citep{balaji2019instance, chen2021robust, huang2020self} implicitly avoid fitting the hard adversarial input-target pairs, by either adaptive inputs or adaptive targets.
By contrast, the methods that focus on fitting hard adversarial~\citep{zhang2021geometryaware} instances are not truly robust under adaptive attacks~\citep{hitaj2021evaluating}.

\textbf{Contributions.} Our contributions are as follows: 
1) Based on a quantitative metric of instance difficulty, we show that fitting hard adversarial instances leads to degraded generalization performance in adversarial training.
2) We conduct rigorous theoretical analyses on both linear and nonlinear models. For linear models, we show analytically that models trained on harder instances have larger robust test error than the ones trained on easy instances; the gap increases with the size of the adversarial budget. For nonlinear models, we derive a lower bound of the model's Lipschitz constant. It increases with the difficulty of the training instances and the size of the adversarial budget, indicating that both factors exacerbate adversarial overfitting.
3) We show that existing approaches to mitigating adversarial overfiting implicitly avoid fitting hard adversarial instances.

\textbf{Notation and terminology.} In this paper, $\vx$ and $\vx'$ are the clean input and its adversarial counterpart.
We use $f_\vw$ to represent a model parameterized by $\vw$ and omit the subscript $\vw$ unless ambiguous.
$\vo = f_\vw(\vx)$ and $\vo' = f_\vw(\vx')$ are the model's output of the clean input and the adversarial input.
$\gL_\vw(\vx, \vy)$ and $\gL_\vw(\vx', \vy)$ represent the loss of the clean and adversarial instances, receptively, in which we sometimes omit $\vw$ and $\vy$ for notation simplicity.
We use $\|\vw\|$ and $\|\rmX\|$ to represent the $l_2$ norm of the vector $\vw$ and the spectral norm of the matrix $\rmX$, respectively.
$sign$ is an elementwise function which returns $+1$ for positive elements, $-1$ for negative elements and $0$ for $0$.
$\mathbf{1}_y$ is the one-hot vector with only the $y$-th dimension being $1$.
The term \textit{adversarial budget} refers to the allowable perturbations applied to the input instance.
It is characterized by $l_p$ norm and the size $\epsilon$ as a set $\Set^{(p)}(\epsilon) = \{\Delta | \|\Delta\|_p \leq \epsilon\}$.
A notation table is provided in Appendix~\ref{sec:notation}.

\change{Based on the notations above, given the training set $\gD$, the robust learning problem can be formulated as the following min-max optimization problem.
Unless explicitly stated, we usually omit $y$ in the loss function for notation simplicity.
\begin{equation}
\begin{aligned}
\min_\vw \E_{(\vx, y) \sim \gD} \max_{\Delta \in \Set^{(p)}(\epsilon)} \gL_\vw(\vx + \Delta, y)
\end{aligned} \label{eq:problem}
\end{equation}}
In this paper, \textit{vanilla training} refers to training on the clean inputs, and \textit{vanilla adversarial training} to the adversarial training method in~\cite{madry2017towards}.
\textit{RN18} and \textit{WRN34} are the 18-layer ResNet~\citep{he2016deep} and the 34-layer WideResNet~\citep{zagoruyko2016wide} with the width factor 10 used in~\cite{madry2017towards} and~\cite{wong2020fast}, respectively.
\revision{To avoid confusion with the general term \textit{overfitting}, which refers to the gap between the training error and the test error, we use the term \textit{adversarial overfitting} to indicate the phenomenon where the robust error on the test set significantly increases in the late phase of training.
Adversarial overfitting often results in a significant generalization gap, because the model's robust error on the training set decreases during training, an increase in robust test error indicates that the model is not effectively generalizing to new data.
} 



The code to reproduce the results of this paper is publicly available on Github\footnote{\href{https://github.com/IVRL/RobustOverfit-HardInstance.git}{https://github.com/IVRL/RobustOverfit-HardInstance.git}}.
\section{Related Work}

We concentrate on white-box attacks, where the attacker has access to the model parameters.
Such attacks are usually based on first-order information and stronger than black-box attacks~\citep{andriushchenko2020square, dong2018boosting}.
For example, the \textit{fast gradient sign method (FGSM)}~\citep{goodfellow2014explaining} perturbs the input based on its gradient's sign.
The \textit{iterative fast gradient sign method (IFGSM)}~\citep{kurakin2016adversarial} iteratively runs FGSM using a smaller step size and projects the perturbation to the adversarial budget after each iteration.
On top of IFGSM, \textit{projected gradient descent (PGD)}~\citep{madry2017towards} uses random initialization and restarts to boost the strength of the attack.

It is challenging to defend models against adversarial examples. Some early defense methods~\citep{pang2019improving,pang2020rethinking,xiao2020enhancing} are shown to utilize obfuscated gradients~\citep{athalye2018obfuscated}, which means they can only tackle some specific types of attacks instead of achieving true robustness.
Models trained by these methods are vulnerable to stronger adaptive attacks~\cite{athalye2018obfuscated, croce2020reliable, tramer2020adaptive, croce2021mind}.
In contrast, several works have designed training algorithms to obtain \textit{provably} robust  models~\citep{raghunathan2018certified, wong2018provable, cohen2019certified, gowal2019scalable, salman2019provably}. 
Unfortunately, these methods either do not generalize to modern network architectures or have a prohibitively large computational complexity.
As a consequence, adversarial training~\citep{madry2017towards} and its variants~\citep{alayrac2019labels, carmon2019unlabeled, hendrycks2019using, sinha2019harnessing, zhang2019you, gowal2020uncovering, wu2020adversarial, 2021Improving, jiang2023towards, Wang2023BetterDM, Cui2024DecoupledKD, zhong2024efficient} have become the de facto approach to obtain robust models in practice.
In essence, these methods generate adversarial examples, usually using PGD, and use them to optimize the model parameters.

While effective, adversarial training is more challenging than vanilla training. 
It was shown to require larger models~\citep{xie2020intriguing} and to exhibit a poorer convergence behavior~\citep{liu2020loss}.
Furthermore, as observed in~\cite{rice2020overfitting}, it suffers from \textit{adversarial overfitting}:
the robust accuracy on the test set significantly decreases in the late adversarial training phase.
\cite{rice2020overfitting} thus proposed to perform early stopping based on a separate validation set to improve the generalization performance in adversarial training.
Furthermore,~\cite{chen2021robust} introduced logit smoothing and weight smoothing strategies to reduce adversarial overfitting.
In parallel to this, several techniques to improve the model's robust test accuracy were proposed~\citep{wang2020improving, wu2020adversarial, zhang2021geometryaware}, but without solving the adversarial overfitting issue.
By contrast, other works~\citep{balaji2019instance, huang2020self} were empirically shown to mitigate adversarial overfitting but without providing any explanations as to how this phenomenon was addressed.

\edit{In addition to adversarial training, there are some previous works studying the training dynamics and generalization properties of vanilla training~\citep{neyshabur2017exploring, zhang2017understanding, toneva2018empirical, swayamdipta2020dataset}.
Unlike adversarial training, models usually have pretty good generalization performance~\citep{bartlett2020benign, li2021towards, kou2023benign} despite over-parameterization, which are usually the cases of deep neural networks.
This phenomenon is called \textit{benign overfitting}.
There are some works connecting benign overfitting with adversarial robustness.
\cite{bubeck2021universal} theoretically proves that at least $\Omega(nm)$ trainable parameters are needed for interpolating $n$ $m$-dimensional instances.
\cite{sanyal2020benign} studies the overparameterization regime in the context of label noise, and demonstrates that label noise in the training data dramatically hurts adversarial robustness.
}

In this paper, we study the causes of adversarial overfitting from both an empirical and a theoretical point of view.
\edit{We address how adversarial perturbations affect the generalization properties of deep neural networks.}
We also identify the reasons why prior attempts~\citep{balaji2019instance, chen2021efficient, huang2020self} successfully mitigate it.

\section{A Metric for Instance Difficulty} \label{sec:hardeasy}

Parametric models are trained to minimize a loss objective based on several input-target pairs called training set, and are then evaluated on a held-out set called test set.
By comparing the loss value of each instance, we can understand which ones, in either the training or the test set, are more difficult for the model to fit.
Therefore, our metric for instance difficulty is based on an instance's loss during the training process.

\revision{To this end, considering that we train the model for $M$ epochs, we use $\{\vw_i\}_{i = 1}^M$ to represent the model parameters in each epoch.
In addition, we introduce the perturbation algorithm $\gA$ and use $\gA(\vx, \vw)$ to denote the adversarial examples of the input $\vx$ given the model parameters $\vw$.
In vanilla training, $\gA_{clean}$ does not perturb the input, i.e., $\gA_{clean}(\vx, \vw) = \vx$; in adversarial training in~\citep{madry2017towards}, $\gA_{PGD}(\vx, \vw)$ is the adversarial example of $\vx$ generated by PGD.
Under this notation, the average loss $\overline{\gL}$ is calculated as $\overline{\gL}(\vx, \gA) \defeq \frac{1}{M}\sum_{i = 1}^M \gL_{\vw_i}(\gA(\vx, \vw_i), y)$, where the loss function $\gL$ is defined in Equation~(\ref{eq:problem}).}
We then study the relative difficulty level of an instance within a finite set, and define the difficulty function $d$ of an instance $\vx$ within a set $\gD$ \revision{for the perturbation algorithm $\gA$} as
\begin{equation}
\begin{aligned}
d(\vx, \gA) = \sP(\overline{\gL}(\vx, \gA) > \overline{\gL}(\widetilde{\vx}, \gA) \vert \widetilde{\vx} \sim U(\gD)) + \frac{1}{2} \sP(\overline{\gL}(\vx, \gA) = \overline{\gL}(\widetilde{\vx}, \gA) \vert \widetilde{\vx} \sim U(\gD))\;,
\end{aligned} \label{eq:difficulty}
\end{equation}
where $\widetilde{\vx} \sim U(\gD)$ indicates that $\widetilde{\vx}$ is uniformly sampled from the finite set $\gD$.
\revision{$d(\vx, \gA)$ is defined based on the model, the attack algorithm $\gA$ and the set $\gD$.
Since $d(\vx, \gA)$ denotes the relative difficulty, it is a bounded function, close to $1$ for the hardest instances and close to $0$ for the easiest ones.}

\revision{We discuss the motivation for and properties of $d(\vx, \gA)$ in Appendix~\ref{subsec:d_function}.
In particular, in Appendix~\ref{subsec:d_function}, we demonstrate that the difficulty function $d$ mainly depends on the original data $\vx$ and the perturbation algorithm $\gA$; the model architecture and the training duration have negligible effects on $d$.
Therefore, we use $\vx$ and $\gA$ as the parameters of the function $d$, and omit the others for notation simplicity.}
In other words, $d(\vx, \gA)$ can represent the difficulty of $\vx$ within a set under a specific type of attack $\gA$.


We show some of the easiest and hardest examples according to our metric in adversarial training in Figure \ref{fig:easy_hard_picture}, \revision{which indicates that our metric aligns well with human perception.}
The easiest instances are visually highly similar, with consistent and typical features of the corresponding category.
By contrast, the hardest ones are much more diverse and with non-typical visual features.
Some of them are ambiguous or even incorrectly labeled.

\revision{
In the remainder of this paper, we use the difficulty metric as defined by Equation (\ref{eq:difficulty}), which not only aligns well with human perception but also is straightforward, easy to obtain, and facilitates our theoretical analysis.
Although other instance difficulty metrics have been proposed, such as the ones in~\cite{baldock2021deep, paul2021deep} based on margins to the decision boundary, comparing them with our metric is subjective and out of the scope of this work.
We focus on using the difficulty metric as a tool to analyze the adversarial overfitting phenomenon.
In the following sections, we study how easy and hard training instances affect adversarial overfitting.
}

\begin{figure}[!t]
\centering
\begin{subfigure}{.23\columnwidth}
\includegraphics[width = \textwidth]{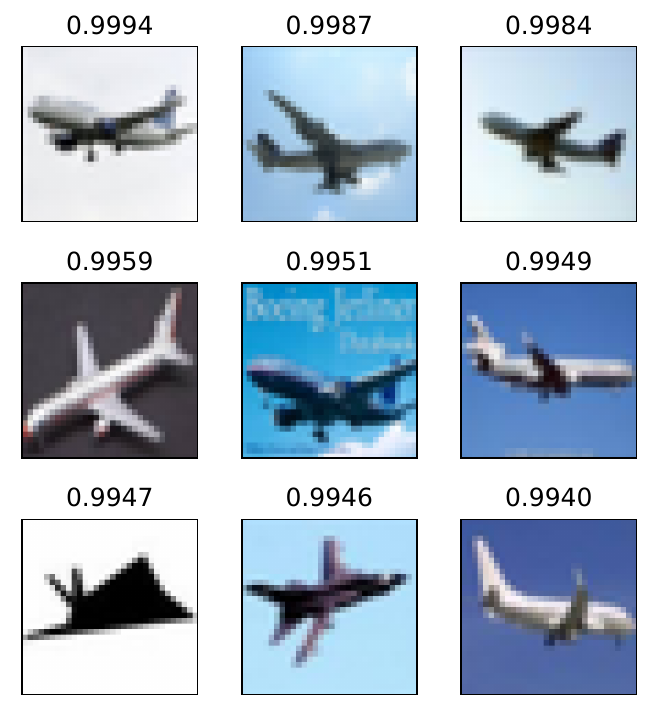}
\caption{Easy@CIFAR10.}
\end{subfigure}
\begin{subfigure}{.23\columnwidth}
\includegraphics[width = \textwidth]{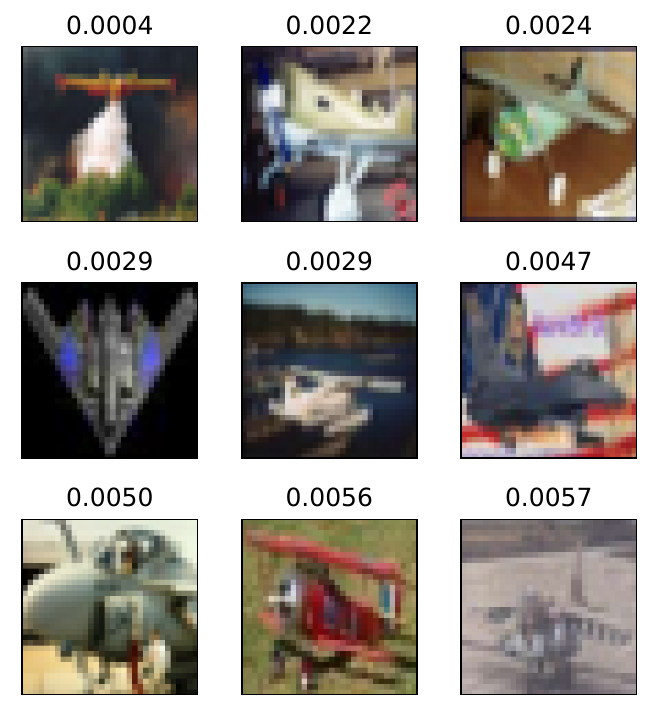}
\caption{Hard@CIFAR10.}
\end{subfigure}
~~
\begin{subfigure}{.23\columnwidth}
\includegraphics[width = \textwidth]{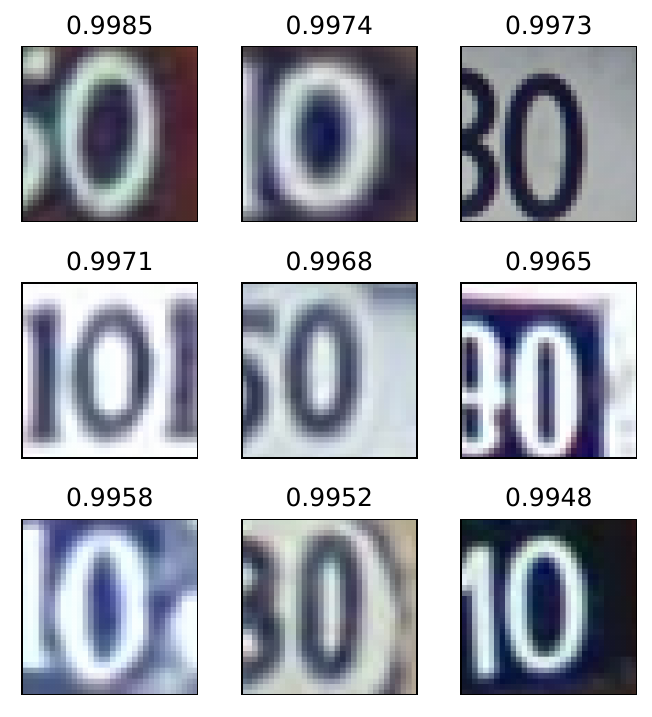}
\caption{Easy@SVHN.}
\end{subfigure}
\begin{subfigure}{.23\columnwidth}
\includegraphics[width = \textwidth]{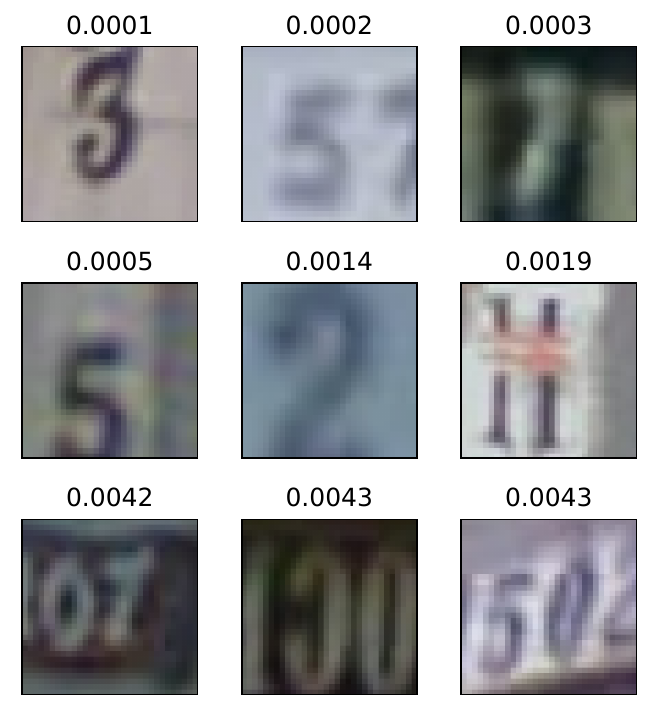}
\caption{Hard@SVHN.}
\end{subfigure}
\caption{Some examples of the easiest and the hardest instances in CIFAR10~\citep{krizhevsky2009learning} and SVHN~\citep{netzer2011reading} datasets. We pick some examples from the ``plane'' category in CIFAR10 and ``0'' category in SVHN. \edit{The number on top of each image indicates the corresponding value of the difficulty function}}
\label{fig:easy_hard_picture}
\end{figure}

\section{\change{Hard Instances Lead to Overfitting}} \label{sec:overfit}

We empirically study how easy and hard instances impact the performance of adversarial training, with a focus on the adversarial overfitting phenomenon.
Unless otherwise mentioned, we use the general experimental settings in Appendix~\ref{sec:app_exp_settings_general}.


\subsection{Using a Subset of Training Data} \label{subsec:subset}

We start by training RN18 models for 200 epochs using either the 10000 easiest, random or hardest instances of the CIFAR10 training set via either vanilla training, FGSM or PGD adversarial training.
For FGSM and PGD adversarial training, the adversarial budget is based on the $l_\infty$ norm and $\epsilon = 8 / 255$.
\revision{Note that the instance's difficulty is defined based on Equation~(\ref{eq:difficulty}) with the same perturbations as in training.
The perturbations of vanilla training are considered to be zero.
In addition, we enforce the training subsets to be class-balanced.}
For example, the easiest 10000 instances consist of the easiest 1000 instances in each class. We provide the learning curves under different perturbations in Figure~\ref{fig:overfit_main}.

\begin{figure}[!ht]
\centering
\begin{subfigure}{0.31\columnwidth}
\includegraphics[width = \textwidth, height = 4cm]{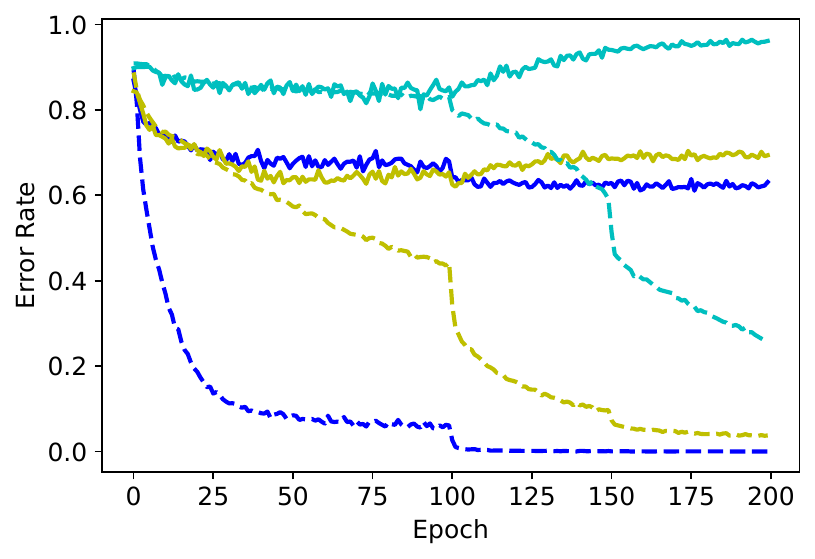}
\caption{PGD Adversarial Training. \label{subfig:overfit_main_pgd}}
\end{subfigure}
\begin{subfigure}{0.34\columnwidth}
\includegraphics[width = \textwidth, height = 4cm]{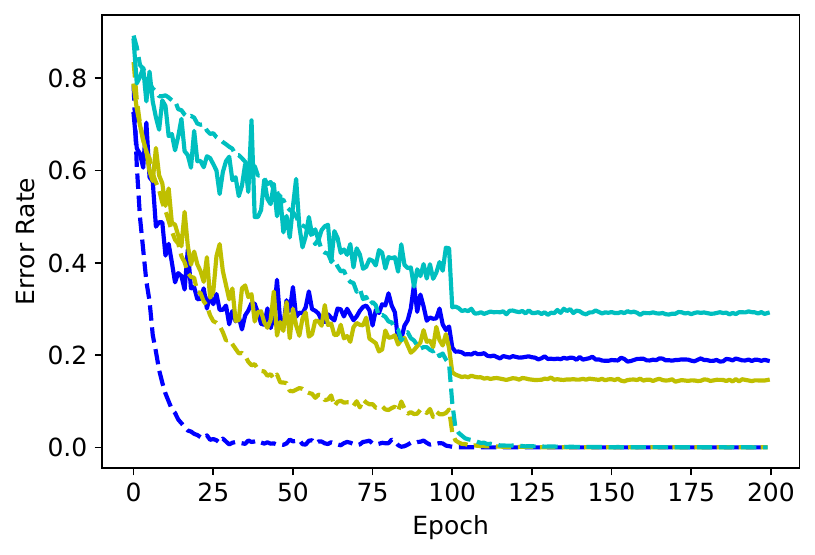}
\caption{FGSM Adversarial Training. \label{subfig:overfit_main_fgsm}}
\end{subfigure}
\begin{subfigure}{0.31\columnwidth}
\includegraphics[width = \textwidth, height = 4cm]{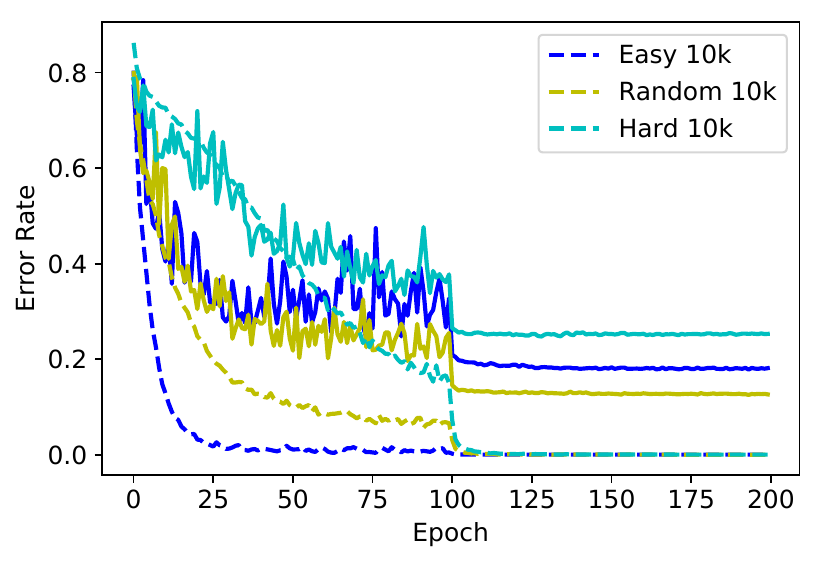}
\caption{Vanilla Training. \label{subfig:overfit_main_clean}}
\end{subfigure}
\caption{Learning curves obtained by training on the $10000$ easiest, random and hardest instances of CIFAR10 under different scenarios. The training error (dashed lines) is the error on the selected instances, and the test error (solid lines) is the error on the whole test set. \revision{The y-axis of each subfigure indicates the training or test error under the corresponding perturbation, so the error rates of different subfigures are not comparable.}} \label{fig:overfit_main}
\end{figure}

\begin{figure}
\centering
\begin{subfigure}{0.31\columnwidth}
\includegraphics[width = \textwidth, height = 4cm]{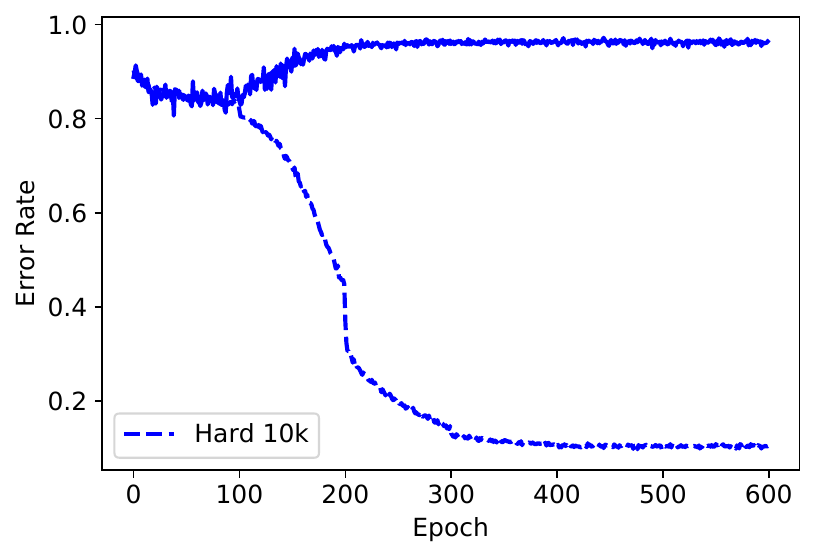}
\caption{Longer Duration. \label{fig:longerduration}}
\end{subfigure}
\begin{subfigure}{0.34\columnwidth}
\includegraphics[width = \textwidth, height = 4cm]{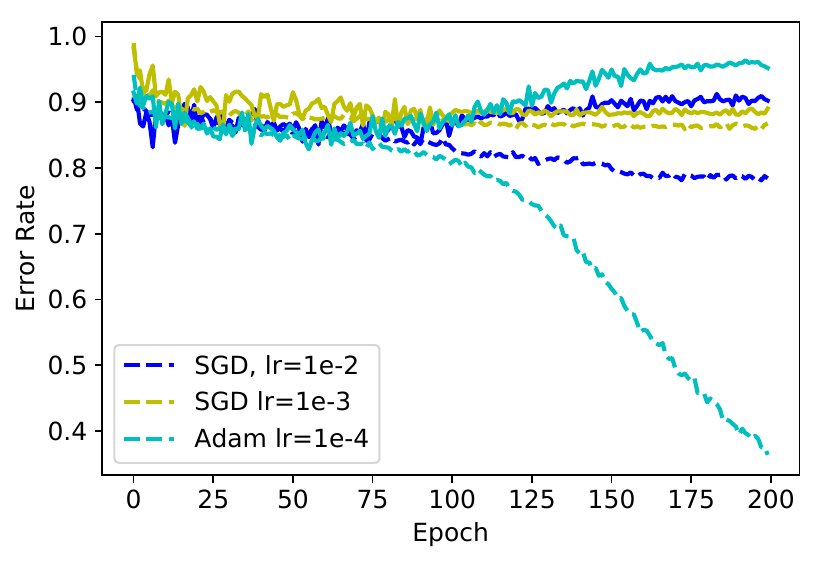}
\caption{Different Optimizers. \label{fig:diffopt}}
\end{subfigure}
\begin{subfigure}{0.31\columnwidth}
\includegraphics[width = \textwidth, height = 4cm]{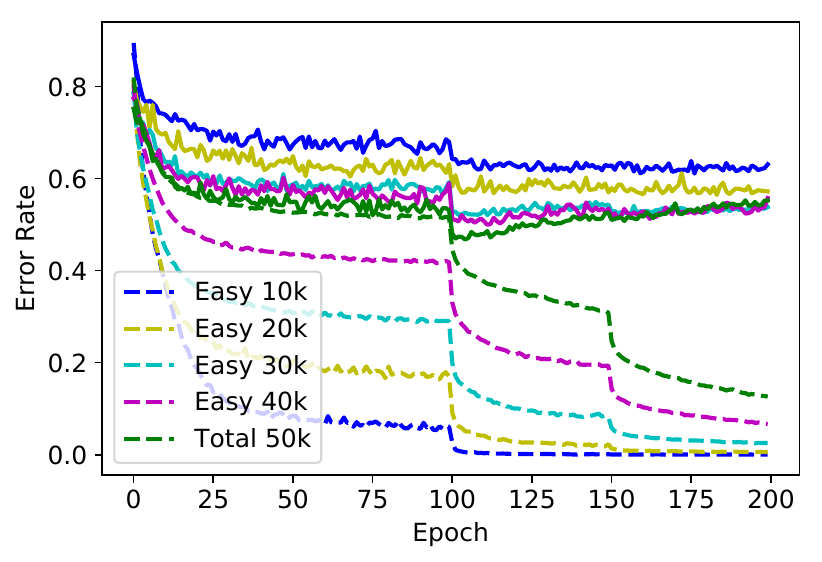}
\caption{Hard Instance Removal. \label{fig:hardremoval}}
\end{subfigure}
\caption{(a) The training error (dashed line) and the test error (solid line) when we conduct adversarial training on the $10000$ hardest training instances for more epochs until convergence. (b) The learning curves of training on the $10000$ hardest training instances when we use a different optimizer, including different learning rates and a different algorithm. (c) The learning curves on the training (dash lines) and the test (solid lines) sets when we remove the hardest training instances.}
\end{figure}

For PGD adversarial training, in Figure~\ref{subfig:overfit_main_pgd}, while we observe adversarial overfitting as in~\cite{rice2020overfitting} when using the random instances, no such phenomenon occurs when using the easiest instances: the performance on the test set does not degrade during training.
However, PGD adversarial training fails and suffers more severe overfitting when using the hardest instances.
Note that this failure is not due to improper optimization.
In Figure~\ref{fig:longerduration} and~\ref{fig:diffopt}, we use longer training duration and different optimizers to conduct PGD adversarial training on the hardest training instances, but the models' performance on the test set are always near trivial.
All these phenomena indicate that the cause of overfitting is fitting the hard adversarial instances generated by PGD.

By contrast, FGSM adversarial training and vanilla training (Figure~\ref{subfig:overfit_main_fgsm},~\ref{subfig:overfit_main_clean}) do not suffer from severe adversarial overfitting.
In these cases, the models trained with the hardest instances also achieve non-trivial test accuracy.
Furthermore, the gaps in robust test accuracy between the models trained by easy instances and by hard ones are much smaller.
\revision{Since vanilla training can be considered as PGD adversarial training with $\epsilon = 0$, FGSM adversarial training does not yield truly robust models~\citep{madry2017towards}; the observations in Figure~\ref{fig:overfit_main} indicate that adversarial overfitting happens when we aim to obtain models robust against an adversarial budget of a large size $\epsilon$.}

In Appendix~\ref{sec:app_exp_subset}, we perform additional and comprehensive experiments, evidencing that our conclusions hold for \revision{various difficulty metrics, datasets and values of $\epsilon$, and for an adversarial budget based on the $l_2$ norm}.
Specifically, we show that more severe adversarial overfitting happens when the size of the adversarial budget $\epsilon$ increases.
That is to say, in term of model generalization, fitting hard training instances is more harmful when we are training against stronger perturbations.

Despite harmful, the experiments discussed below show that simply removing hard instances is not the optimal choice.
In Figure~\ref{fig:hardremoval}, we conduct PGD adversarial training using increasingly more training instances, starting with the easiest ones.
The learning curves on the test set indicate that the models can still benefit from more data, but only when combined with early stopping  used in~\citep{rice2020overfitting}.
It means that the hard instances can still benefit adversarial training, but need to be utilized in an adaptive manner.

\subsection{Using the Whole Training Set} \label{subsec:hardoverfit}

\begin{figure}[!ht]
\centering
\begin{subfigure}{0.42\columnwidth}
\includegraphics[width = \textwidth]{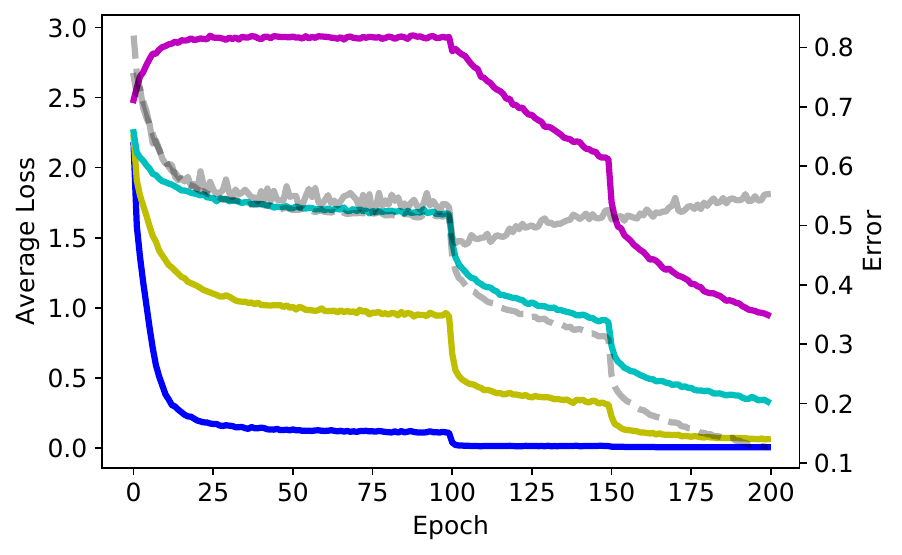}
\caption{\label{subfig:loss_by_group}}
\end{subfigure}
~~~~~~~~~~~~~
\begin{subfigure}{0.42\columnwidth}
\includegraphics[width = \textwidth]{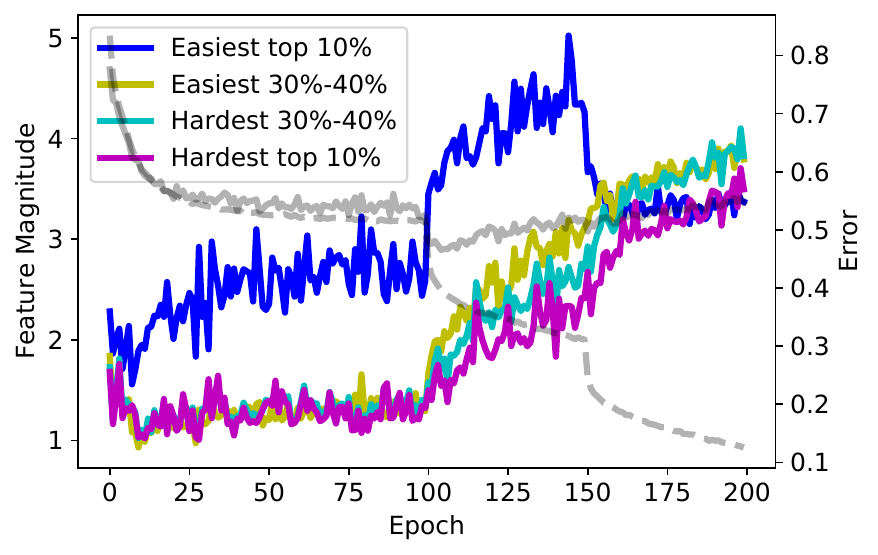}
\caption{\label{subfig:feature_by_group}}
\end{subfigure}
\caption{Analysis on the groups $\gG_0$, $\gG_3$, $\gG_6$ and $\gG_9$ in the training set. The right vertical axis corresponds to the training (dashed grey line) and test (solid grey line) error under adversarial attacks for both plots. \textbf{Left plot:} The left vertical axis represents the average loss of different groups. \textbf{Right plot:} The left vertical axis represents the average $l_2$ norm of features extracted during training for different groups.} \label{fig:loss_by_group}
\end{figure}

Let us now turn to the more standard setting where we train the model with the entire training set.
To nonetheless analyze the influence of instance difficulty in this scenario, we divide the training set $\gD$ into $10$ non-overlapping groups $\{\gG_i\}_{i = 0}^9$, with $\gG_i = \{\vx \in \gD | 0.1 \times i \leq d(\vx, \gA_{PGD}) < 0.1 \times (i + 1)\}$, \revision{where $d(\vx, \gA_{PGD})$ is the difficulty of $\vx$ based on PGD attacks}.
That is, $\gG_0$ is the easiest group, whereas $\gG_9$ is the hardest one. We then train a RN18 model on the entire CIFAR10 training set by PGD adversarial training and monitor the training behavior of the different groups.
In particular, in Figure~\ref{subfig:loss_by_group}, we plot the average loss of the instances in the groups $\gG_0$, $\gG_3$, $\gG_6$ and $\gG_9$.
The results show that, in the early training stages, the model first fits the easy instances, as evidenced by the average loss of group $\gG_0$ decreasing much faster than that of the other groups.
By contrast, in the late training phase, the model tries to fit the more difficult instances, with the average loss of groups $\gG_9$ and $\gG_6$ decreasing much faster than that of the other groups.
In this period, however, the robust test error (solid grey line) increases, which indicates that adversarial overfitting arises from the model's attempt to fit the hard adversarial instances.

In addition to average losses, inspired by~\cite{ilyas2019adversarial}, which showed that the penultimate layer's activations of a robust model correspond to its \textit{robust features} that cannot be misaligned by adversarial attacks, we monitor the group-wise average magnitudes of the penultimate layer's activations.
As shown in Figure~\ref{subfig:feature_by_group}, the model first focuses on extracting robust features for the easy instances, as evidenced by the comparatively large activations of the instances in $\gG_0$.
In the late phase of training, the norm of the activations of the hard instances increases significantly, bridging the gap between easy and hard instances.
This further indicates that the model focuses more on the hard instances in the later phase, at which point it starts overfitting.

\section{Theoretical Analysis} \label{sec:thm}

The empirical study in Section~\ref{sec:overfit} indicates that adversarial overfitting arises from fitting hard adversarial training instances.
We now study this relationship from a theoretical viewpoint.
We start with a linear model: the logistic regression model on a Gaussian Mixture Model. In this toy example, the adversarial examples and the corresponding loss function have analytical expressions, allowing us to precisely draw the relationship between the instance difficulty and the generalization performance.
We then generalize our analysis to general nonlinear models and use the models' Lipschitz constant as a proxy for their robustness on the test set.
Our conclusions are consistent with the empirical study.

We use $\{\vx_i, y_i\}_{i=1}^n$ to represent the training data, and $(\rmX, \vy)$ as its matrix form.
$\{\vx'_i, y_i\}_{i=1}^n$ and $(\rmX', \vy)$ are their adversarial counterparts.
Here, $\vx_i \in \sR^m$, $y_i \in \{-1, +1\}$, $\rmX \in \sR^{n \times m}$ and $\vy \in \{-1, +1\}^n$.
\revision{Note that these adversarial examples are generated based on the model parameters $\vw$ to maximize the loss objective, so they depend on the model parameters $\vw$ and are generated on the fly during training, which is consistent with adversarial training in practice.
For simplicity, we do not explicitly represent this dependence in the notation.}

The notation is summarized in Table~\ref{tbl:notation} of Appendix~\ref{sec:notation}.

\subsection{Linear Models} \label{sec:linear}

We study the logistic regression model under an $l_2$ norm based adversarial budget.
In this case, the model is parameterized by $\vw \in \sR^m$ and outputs $sign(\vw^T\vx'_i)$ given the adversarial example $\vx'_i$ of the input $\vx_i$.
The loss function for this instance is $\frac{1}{1 + e^{y_i \vw^T\vx'_i}}$.
We assume over-parameterization, which means $n < m$.

The following theorem shows that, under mild assumptions, the parameters of the adversarially trained model converge to the $l_2$ max-margin direction of the training data.

\begin{theorem} \label{thm:converge}
For a dataset $\{\vx_i, y_i\}_{i = 1}^n$ that is linearly separable under the adversarial budget $\gS^{(2)}(\epsilon)$, any initial point $\vw_0$ and step size $\alpha \leq 2\|\rmX\|^{-2}$, the gradient descent $\vw_{u + 1} = \vw_u - \alpha \triangledown_{\vw} \gL_{\vw_u}(\rmX')$ converges asymptotically to the $l_2$ max-margin vector of the training data. That is,
\begin{equation}
\begin{aligned}
\lim_{u \to \infty} \frac{\vw_u}{\|\vw_u\|} &= \frac{\widehat{\vw}}{\|\widehat{\vw}\|},\ \mathrm{where}\ \ \widehat{\vw} = \argmin_{\vw} \|\vw\| \\
&s.t. \ \ \forall i \in \{1, 2, ..., n\}, \ \vw^T \vx_i \geq 1\;.
\end{aligned} \label{eq:max_margin}
\end{equation}
\end{theorem}

The proof is in Appendix~\ref{sec:proof_converge}.
Theorem~\ref{thm:converge} extends the conclusion in~\cite{soudry2018implicit}, which only studies the non-adversarial case.
It also indicates that the optimal parameters are only determined by the support vectors of the training data, which are the ones with the smallest margin.
According to the loss function, the smallest margin means the largest loss values and thus the hardest training instances based on our definition in Section~\ref{sec:hardeasy}.

To further study how the training instances' difficulty influences the model's generalization performance, we assume that the data points are drawn from a $K$-mode Gaussian mixture model (GMM).
Specifically, the $k$-th component has a probability $p_k$ of being sampled and is formulated as:
\begin{equation}
\begin{aligned}
\vx_i \sim \gN(y_i r_k \veta, \rmI)
\end{aligned} \label{eq:gmm}
\end{equation}
\change{Here, $\veta \in \sR^m$ is the unit vector indicating the direction of the mean for each mode}, and $r_k \in \sR^+$ controls the average distance between the positive and negative instances.
The mean values of all modes in this GMM are colinear, so $r_k$ indicates the difficulty of instances sampled from the $k$-th component.
\edit{In Appendix~\ref{subsec:consistent_difficulty}, we demonstrate the strong correlation of $r_k$ and the difficulty defined in Section~\ref{sec:hardeasy}.}

Without the loss of generality, we assume that $r_1 < r_2 < ... < r_{K-1} < r_K$.
Same as in Section~\ref{subsec:subset}, we consider models trained with the subsets of the training data, e.g., $n$ instances from the $l$-th component.
$l = 1$ then indicates training on the hardest examples, while $l = K$ means using the easiest.
In matrix form, we have $\rmX = r_l \vy \veta^T + \rmQ$ for the instances sampled from the $l$-th component, where the rows of noise matrix $\rmQ$ are sampled from $\gN(\mathbf{0}, \rmI)$.

Although the max-margin direction in Equation (\ref{eq:max_margin}), where the parameters converge based on Theorem~\ref{thm:converge}, does not have an analytical expression, the results in~\cite{wang2020benign} indicate that, in the over-parameterization regime and when the training data is sampled from a GMM, the max-margin direction is the min-norm interpolation of the data with high probability.
Since the latter has an analytical form given by $\rmX^T (\rmX \rmX^T)^{-1}\vy$, we can then calculate the exact generalization performance of the trained model as stated in the following theorem.

\begin{theorem} \label{thm:main}
If a logistic regression model is adversarially trained on $n$ separable training instances sampled from the $l$-th component of the GMM described in (\ref{eq:gmm}), \revision{$\{p_k\}_{k = 1}^K$ are the probabilities of sampling from the $k$-th component of the GMM};
when $\frac{m}{n\log n}$ is sufficiently large\footnote{Specifically, $m$ and $n$ need to satisfy $m > 10 n \log n + n -1$ and $m > C n r_l \sqrt{\log 2 n}\|\veta\|$. The constant $C$ is derived in the proof of Theorem 1 in~\cite{wang2020benign}.}, then with probability $1 - O(\frac{1}{n})$, the expected adversarial test error $\gR$ under the adversarial budget $\gS^{(2)}(\epsilon)$, which is a function of $r_l$ and $\epsilon$, on the whole GMM described in (\ref{eq:gmm}) is given by
\begin{equation}
\begin{aligned}
\gR(r_l, \epsilon) &= \sum_{k = 1}^K p_k \Phi\left( r_k g(r_l) - \epsilon \right) \\
\mathrm{where}\ g(r_l) &= (C_1 - \frac{1}{C_2 r_l^2 + o(r_l^2)})^{\frac{1}{2}},\ C_1, C_2 \geq 0.
\end{aligned} \label{eq:r}
\end{equation}
$C_1$, $C_2$ are non-negative numbers independent of $\epsilon$ and $r_l$. The function $\Phi$ is defined as $\Phi(x) = \sP(Z > x),\ Z \sim \gN(0, 1)$.
\end{theorem}

We defer the proof of Theorem~\ref{thm:main} to Appendix~\ref{sec:proof_main}, in which we calculate the \textit{exact} expression of $\gR(r_l, \epsilon)$, $C_1$, $C_2$, and show that $C_1$, $C_2$ are positive numbers almost surely.
Since $C_1$ and $C_2$ are independent of $r_l$, and $\Phi(x)$ is a monotonically decreasing function, we  conclude that the robust test error $\gR(r_l, \epsilon)$ becomes smaller when $r_l$ increases.
\change{Since the training set is separable, our results indicate that when the training instances become easier, the corresponding generalization error under adversarial attack becomes smaller.}

\change{Theorem~\ref{thm:main} holds for any $\epsilon$ as long as the training data is separable under the corresponding adversarial budget.}
The following corollary shows that the difference in the robust test error between models trained with easy instances and the ones with hard ones increases when $\epsilon$ becomes larger, i.e., under a larger adversarial budget.
\begin{corollary} \label{coro:epsilon}
Under the conditions of Theorem~\ref{thm:main} and the definition of $\gR$ in Equation (\ref{eq:r}), if $\epsilon_1 < \epsilon_2$, then we have $\forall\ 0 \leq i < j \leq K, \gR(r_i, \epsilon_1) - \gR(r_j, \epsilon_1) < \gR(r_i, \epsilon_2) - \gR(r_j, \epsilon_2)$.
\end{corollary}
The proof is in Appendix~\ref{sec:proof_coro}.
$\gR(r_i, \epsilon) - \gR(r_j, \epsilon)$ is the gap in robust accuracy between the models trained on the easy instances and the ones on the hard instances under the adversarial budget $\gS^{(2)}(\epsilon)$.
Corollary~\ref{coro:epsilon} shows that such a gap increases with the size of the adversarial budget.
\change{This indicates that, compared with training on the clean inputs, i.e., $\epsilon = 0$, the generalization performance of adversarial training, i.e., $\epsilon > 0$, is more sensitive to the difficulty of the training instances.
Furthermore, overfitting in adversarial training becomes increasingly severe as $\epsilon$ becomes larger.
This is consistent with our empirical observations in Figures~\ref{fig:overfit_main},~\ref{fig:overfit_adv_budget},~\ref{fig:overfit_l2_adv_budget}.}

\subsection{General Nonlinear Models} \label{sec:thm_nonlinear}

In this section, we study the binary classification problem using a general nonlinear model.
We consider a model with $b$ parameters, i.e., $\vw \in \sR^b$.
Without loss of generality, we assume the output of the function $f_\vw$ to lies in $[-1, +1]$.
\change{Similarly to the $K$-mode Gaussian mixture model studied in the linear case, we assume the data distribution to be a composition of $K$ sub-distributions.
Furthermore, each of these distributions are isoperimetric.}

\begin{assumption} \label{asp:iso}
The data distribution $\mu$ is a composition of $K$ $c$-isoperimetric distributions on $\sR^m$, each of which has a positive conditional variance.
That is, $\mu = \sum_{k = 1}^K \alpha_k \mu_k$, where $\alpha_k > 0$ and $\sum_{k = 1}^K \alpha_k = 1$.
We define $\sigma^2_k = \E_{\mu_k}[Var[y|\vx]]$, and without loss of generality assume that $\sigma_1 \geq \sigma_2 \geq ... \geq \sigma_K > 0$.
Furthermore, given any $L$-Lipschitz function $f_\vw$, i.e., $\forall \vx_1, \vx_2, \|f_\vw(\vx_1) - f_\vw(\vx_2)\| \leq L\|\vx_1 - \vx_2\|$, we have the following inequality satisfied $\forall k \in \{1,..., K\}$
\begin{equation}
\begin{aligned}
\sP(\vx \sim \mu_k, \|f_\vw(\vx) - \E_{\mu_k}(f_\vw)\| \geq t) \leq 2 e^{-\frac{mt^2}{2cL^2}}\;.
\end{aligned}
\end{equation}
\end{assumption}

This is a benign assumption; the data distribution is a mixture of $K$ components and each of them contains samples from a sub-Gaussian distribution.
These components correspond to training instances of different difficulty levels measured by the conditional variance.
\edit{This is because the conditional variance $\sigma_k^2$ is the expected test error of a well-trained model~\citep{bubeck2021universal}. Subsets with large $\sigma_k^2$ have higher loss and the difficulty function defined by the average training loss}

We now study the properties of the model $f_\vw$ under adversarial attacks.

\begin{definition} \label{def:h}
Given the dataset $\{\vx_i, y_i\}_{i = 1}^n$, the model $f_\vw$, the adversarial budget $\gS^{(p)}(\epsilon)$ and a positive constant $C$, we define the function $h(C, \epsilon)$ as:
\begin{equation}
\begin{aligned}
h(C, \epsilon) &= \min_{\vw \in \gT(C, \epsilon)} \min_i h_{i, \vw}(\epsilon) \\ 
\mathrm{where}\ \gT(C, \epsilon) &= \left\{\vw \bigg\vert \frac{1}{n} \sum_{i = 1}^n (f_\vw(\vx_i') - y_i)^2 \leq C\right\}\;, \\
\ h_{i, \vw}(\epsilon) = \max \zeta,\ &s.t.\ [f_\vw(\vx_i) - \zeta, f_\vw(\vx_i) + \zeta] \subset 
\left\{f_\vw(\vx_i + \Delta) \bigg\vert \Delta \in \gS^{(p)}(\epsilon)\right\}.
\end{aligned}
\end{equation}
Here, $\vx'_i$ is the adversarial example of $\vx_i$.
We omit the superscript $(p)$ for notation simplicity.
\end{definition}

By definition, $h_{i, \vw}(\epsilon) \geq 0$ depicts the bandwidth $\zeta$ of the model's output range in the domain of the adversarial budget on a training instance.
\change{$\gT(C, \epsilon)$ represents the set of well-trained models whose adversarial training loss is smaller than $C$.
Therefore, $h(C, \epsilon)$ is the minimum bandwidth among such well-trained models.
The following lemma demonstrates monotonicity properties of the function $h$.}

\begin{lemma} \label{lem:c}
$\forall C, \epsilon_1 < \epsilon_2$, $h(C, \epsilon_1) \leq h(C, \epsilon_2)$; $\forall \epsilon, C_1 < C_2$, $h(C_1, \epsilon) \geq h(C_2, \epsilon)$.
\end{lemma}

Based on the definitions of $\gT$ and $h_{i, \vw}$, and for a fixed value of $C$, we have $\forall \epsilon_1 < \epsilon_2$, $h_{i, \vw}(\epsilon_1) \leq h_{i, \vw}(\epsilon_2)$ and $\gT(C, \epsilon_2) \subset \gT(C, \epsilon_1)$.
As a result, $\forall \epsilon_1 < \epsilon_2$, $h(C, \epsilon_1) \leq h(C, \epsilon_2)$.
In addition, since $\forall C_1 < C_2$, $\gT(C_1, \epsilon) \subset \gT(C_2, \epsilon)$ for a fixed value of $\epsilon$, we have $\forall C_1 < C_2$, $h(C_1, \epsilon) \geq h(C_2, \epsilon)$.
That is to say, $h(C, \epsilon)$ is a monotonically non-decreasing function on $\epsilon$ and a monotonically non-increasing function on $C$.
In practice, when $f_\vw$ represents a deep neural network, $h(C, \epsilon)$ increases with $\epsilon$ almost surely, because the attack algorithm usually generates adversarial examples at the boundary of the adversarial budget.
Based on the monotonicity properties of $h$, We then state our main theorem below.

\begin{theorem} \label{thm:nonlinear}
Given $n$ training pairs $\{\vx_i, y_i\}_{i = 1}^n$ sampled from the $l$-th component $\mu_l$ of the distribution in Assumption~\ref{asp:iso}, the parametric model $f_\vw$, the adversarial budget $\gS^{(p)}(\epsilon)$ and the corresponding function $h$ defined in Definition~\ref{def:h}, we assume that the model $f_\vw$ is in the function space $\gF = \{f_\vw, \vw \in \gW\}$ with $\gW \subset \sR^b$ having a finite diameter $diam(\gW) \leq W$ and, $\forall \vw_1, \vw_2 \in \gW, \|f_{\vw_1} - f_{\vw_2}\|_\infty \leq J\|\vw_1 - \vw_2\|_\infty$.
We train the model $f_\vw$ adversarially using these $n$ data points.
\revision{Let $\vx'_i$ be the adversarial example of the data point $\vx_i$, i.e., $\vx'_i = \argmax_{\vx_{adv}} (f_\vw(\vx_{adv}) - y_i)^2$ s.t. $\|\vx_{adv} - \vx_i\|_p \leq \epsilon$.}
$\forall \delta \in (0, 1)$, if we have $\frac{1}{n} \sum_{i = 1}^n (f_\vw(\vx'_i) - y_i)^2 = C$ and $\gamma \defeq \sigma^2_l  + h^2(C, \epsilon) - C \geq 0$, then with probability at least $1 - \delta$, the Lipschitz constant of $f_\vw$ is lower bounded as
\begin{equation}
\begin{aligned}
Lip(f_\vw) \geq \frac{\gamma}{2^7}\sqrt{\frac{nm}{c\left(b\log(4WJ\gamma^{-1}) - \log(\delta/2 - 2e^{-2^{-11}n\gamma^2})\right)}}\;,
\end{aligned} \label{eq:lip_bound}
\end{equation}
$Lip(f_\vw)$ is the Lipschitz constant of $f_\vw$: $\forall \vx_1, \vx_2$, $\|f_\vw(\vx_1) - f_\vw(\vx_2)\| \leq Lip(f_\vw)\|\vx_1 - \vx_2\|$.
\end{theorem}

The proof is deferred to Appendix~\ref{sec:nonlinear_proof}.
Theorem~\ref{thm:nonlinear} extends the results in~\cite{bubeck2021universal} to the case of adversarial training.
The Lipschitz constant is widely used to bound a model's adversarial vulnerability on the test set~\citep{ruan2018reachability, weng2018towards, weng2018evaluating}; larger Lipschitz constants indicate higher adversarial vulnerability on the test set.
Note that modern deep neural network models typically have millions of parameters, so $b \gg \max\{c, m, n\}$.
In this case, we can approximate the lower bound (\ref{eq:lip_bound}) by $Lip(f_\vw) \gtrsim \frac{\gamma}{2^7}\sqrt{\frac{nm}{bc\log(4WJ\gamma^{-1})}}$, and the right hand side increases with $\gamma$.


Lemma~\ref{lem:c} indicates that $\gamma$ monotonically increases with the decrease of $C$, and Theorem~\ref{thm:nonlinear} assumes $\gamma > 0$, so the conclusion of Theorem~\ref{thm:nonlinear} is based on a sufficient small adversarial training loss $C$.
That is to say, our theorem is applicable when the model is well fit to the adversarial training instances, \revision{i.e., small adversarial training loss}, which is exactly when adversarial overfitting occurs.
\revision{By contrast, there is usually no adversarial overfitting with large adversarial training loss when adversarial training does not or cannot fit the training set.
For example, when $\epsilon$ is too large for adversarial training to converge, we will obtain a constant classifier as indicated in~\cite{liu2020loss}.
While the model has a high robust test error, the adversarial overfitting does not happen in this case.}

\change{Theorem~\ref{thm:nonlinear} is applicable to any $l_p$ norm based adversarial budget based on the definition of $h(C, \epsilon)$}.
Since $\gamma \defeq \sigma^2_l  + h^2(C, \epsilon) - C$, we can conclude that the Lipschitz upper bound and thus the adversarial vulnerability on the test set is affected by three factors: it increases when $\sigma_l$ , $\epsilon$ increase or $C$ decreases. We elaborate the conclusion in the following paragraphs.

First, as the training processes, the adversarial training loss $C$ becomes smaller, and correspondingly $h(C, \epsilon)$ becomes bigger based on Lemma~\ref{lem:c}. Therefore, $\gamma = \sigma_l^2 + h^2(C, \epsilon) - C$ increases and the condition $\gamma \geq 0$ will be satisfied in the late phase of adversarial training.
In this context, as $\gamma \geq 0$ increases during this period, the Lipschitz lower bound also increases based on (\ref{eq:lip_bound}), indicating a higher adversarial test loss.
In summary, in the final stages of training, which ensure that $\gamma \geq 0$, the training loss $C$ decreases while the test loss increases.
As a result, the generalization gap increases.

\change{Second, with fixed $C$, i.e., the adversarial training loss is fixed, and the generalization gap is indicated by the adversarial test loss, represented by the Lipschitz lower bound in (\ref{eq:lip_bound}).
When $\epsilon$ is fixed, the Lipschitz lower bound increases with the increase of $\sigma_l$.
That is, under the same adversarial budget, the generalization gap increases with the instances' difficulty, measured by $\sigma_l$ in our theorem.
When $\sigma_l$ is fixed, the Lipschitz lower bound increases with the increase of $\epsilon$.
Therefore, using the same training instances, the generalization gap increases with the size of the adversarial budget, measured by $\epsilon$.}

Finally, Theorem~\ref{thm:nonlinear} discusses the case where the model is trained on samples from one components of the data distribution, i.e., a subset of the training set.
This is exactly the case of Section~\ref{subsec:subset}.
Furthermore, we can utilize Theorem~\ref{thm:nonlinear} to analyze the cases when the model is trained on samples from the entire data distribution, which consists from $K$ components.
Similarly to the analysis in Section~\ref{subsec:hardoverfit}, we calculate the training loss $\{C_i\}_{i = 1}^K$ for all $K$ components.
Correspondingly, we can define the function $h_i(C, \epsilon)$ same as in Definition~\ref{def:h} except that it is based on, instead of all training instances, the training instances sampled from the $i$-th component from the data distribution.
Based on this, we define $\gamma_i \defeq \sigma_i^2 + h^2_i(C_i, \epsilon) - C_i$ for $i \in \{1, 2, ..., K\}$.
We can then utilize Theorem~\ref{thm:nonlinear} for training samples from each distribution component and then obtain the lower bound of the model's Lipschitz constant.
Formally, we have the following: 
\begin{corollary} \label{coro:wholeset}
Given the assumptions of Theorem~\ref{thm:nonlinear}, except that the training data is sampled from all $K$ components and contains $n_i$ training instances from the $i$-th component, $\{C_i\}_{i = 1}^K$, $\{h_i\}_{i = 1}^K$, $\{\gamma_i\}_{i = 1}^K$ defined for each components of the data distribution, then with probability at least $1 - \delta$, the Lipschitz constant of $f_\vw$ is lower bounded as 
\begin{equation}
\begin{aligned} \label{eq:wholeset}
Lip(f_\vw) \geq \max \left\{ \frac{\gamma_i}{2^7}\sqrt{\frac{n_im}{c\left(b\log(4WJ\gamma_i^{-1}) - \log(\delta/2 - 2e^{-2^{-11}n\gamma_i^2})\right)}} \bigg\vert \gamma_i \geq 0 \right\}
\end{aligned}
\end{equation}
\end{corollary}

Corollary~\ref{coro:wholeset} is straightforward from Theorem~\ref{thm:nonlinear}: We calculate the Lipschitz lower bound based on the adversarial training loss of each component as long as it is valid, i.e., $\gamma_i \geq 0$.
The formal proof is provided in Appendix~\ref{subsec:wholeset}.
Corollary~\ref{coro:wholeset} indicates the Lipschitz lower bound of the model when it is trained on the whole training distribution consisting of instances of different difficulty levels.
Similarly to the analysis of Theorem~\ref{coro:wholeset}, the value of $\gamma_i$ for each component of the data distribution increases as the training processes.
That is to say, the size of the set $\{i \vert \gamma_i \geq 0\}$ increases during training, i.e., there are more and more numbers fed to the max operator in (\ref{eq:wholeset}).
In addition, the Lipschitz lower bound derived by the training instances from each components monotonically increases during training.
Combining these two points together, we conclude that the Lipschitz lower bound provided by (\ref{eq:wholeset}) monotonically increases during training, indicating more and more severe overfitting.
As in Theorem~\ref{thm:nonlinear}, the Lipschitz lower bound also increases with the increase of $\epsilon$, indicating that using a larger adversarial budget in adversarial training suffers more from overfitting.

In the early phase of adversarial training, the difference in the adversarial training loss for easy and hard instances is large.
That is, the value of $C_i$ dominates the calculation of $\gamma_i$.
In this stage, the Lipschitz lower bound in (\ref{eq:wholeset}) is dominated by the easy instances, because for hard instances sampled from the $i$-th component, $C_i$ is huge and the corresponding $\gamma_i$ does not satisfy the condition $\gamma_i \geq 0$.
However, in the late phase of adversarial training, the adversarial training loss for all training instances is close to $0$.
As a result, the value of $\sigma_i$ dominates the calculation of $\gamma_i$.
In this stage, the Lipschitz lower bound in (\ref{eq:wholeset}) is dominated by the hard instances, because $\forall i, C_i \simeq 0$, and $\gamma_i$ increase with the increase of $\sigma_i$.

Corollary~\ref{coro:wholeset} explains the phenomena shown in Section~\ref{subsec:hardoverfit} and confirms that fitting hard adversarial instances is harmful for the generalization performance. 


\subsection{Numerical Simulation} \label{subsec:simulation}

We conduct numerical simulation to confirm the validity of Theorem~\ref{thm:nonlinear} in our settings.
To this end, we use the CIFAR10 dataset and an RN18 network architecture.
\change{However, calculating the Lipschitz constant of a deep neural network is NP-hard~\citep{scaman2018lipschitz}, exactly calculating the Lipschitz constant~\citep{jordan2020exactly} is so far infeasible for modern deep neural networks.
Instead, we therefore estimate the upper bound of the Lipschitz constant numerically, as in~\citep{scaman2018lipschitz}.
}


\begin{table}[!ht]
\centering
\begin{tabular}{p{2.5cm}p{2.8cm}<{\centering}p{2.8cm}<{\centering}p{2.8cm}<{\centering}}
\Xhline{4\arrayrulewidth}
\multirow{2}{*}{Value of $\epsilon$}  & \multicolumn{3}{c}{Lipschitz in $l_\infty$ Cases ($\times 10^4$)} \\
 & $2 / 255$ & $4 / 255$ & $8 / 255$ \\
\hline
Easy10K     & $5.91 \pm 0.00$  & $6.06 \pm 0.00$   & $14.54 \pm 0.02$  \\
Random10K   & $28.98 \pm 0.03$ & $79.96 \pm 0.16$  & $93.63 \pm 0.34$  \\
Hard10K     & $72.42 \pm 0.48$ & $117.60 \pm 2.18$ & $567.24 \pm 0.59$ \\
\Xhline{4\arrayrulewidth}
\end{tabular}

\begin{tabular}{p{2.5cm}p{2.8cm}<{\centering}p{2.8cm}<{\centering}p{2.8cm}<{\centering}}
\Xhline{4\arrayrulewidth}
\multirow{2}{*}{Value of $\epsilon$}  & \multicolumn{3}{c}{Lipschitz in $l_2$ Cases ($\times 10^4$)} \\
 & $0.50$ & $0.75$ & $1.00$ \\
\hline
Easy10K    & $3.34 \pm 0.01$  & $3.67 \pm 0.00$  & $3.91 \pm 0.00$  \\
Random10K  & $30.01 \pm 0.08$ & $31.28 \pm 0.04$ & $39.34 \pm 0.08$ \\
Hard10K    & $60.62 \pm 0.07$ & $80.06 \pm 0.16$ & $77.55 \pm 0.61$ \\
\Xhline{4\arrayrulewidth}
\end{tabular}
\caption{Upper bound of the Lipschitz constant under different settings of the adversarial budget when the model is adversarially trained for the easiest, random, or the hardest $10000$ instances of the training set. \revision{The experiments are run $20$ times, and we report both the mean and the standard deviation in the form of ``mean $\pm$ standard deviation''.}} \label{tbl:upper_lip}
\end{table}

Table~\ref{tbl:upper_lip} provides the upper bound of the Lipschitz constant of models trained by different subsets of the training data and different adversarial budget.
Due to the stochasticity introduced by the algorithm of~\cite{scaman2018lipschitz}, we run it $20$ times and report the average \revision{and standard deviation; we observed that the standard deviation is negligible compared with the magnitude of the mean value}.
Based on the results in Table~\ref{tbl:upper_lip}, it is clear that the models adversarially trained on the hard training instances have a much larger Lipschitz constant than the ones trained on the easy instances.

\begin{figure}[h]
\centering
\includegraphics[width = 0.4\textwidth]{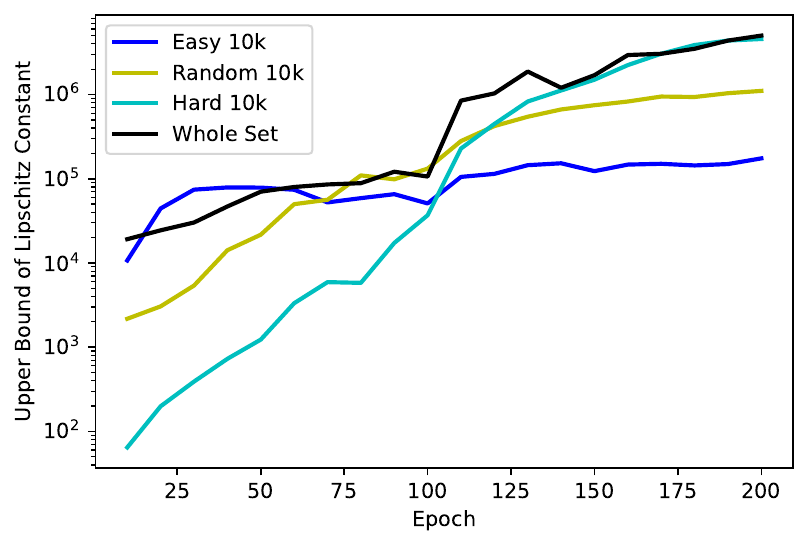}
~~~~~~
\includegraphics[width = 0.4\textwidth]{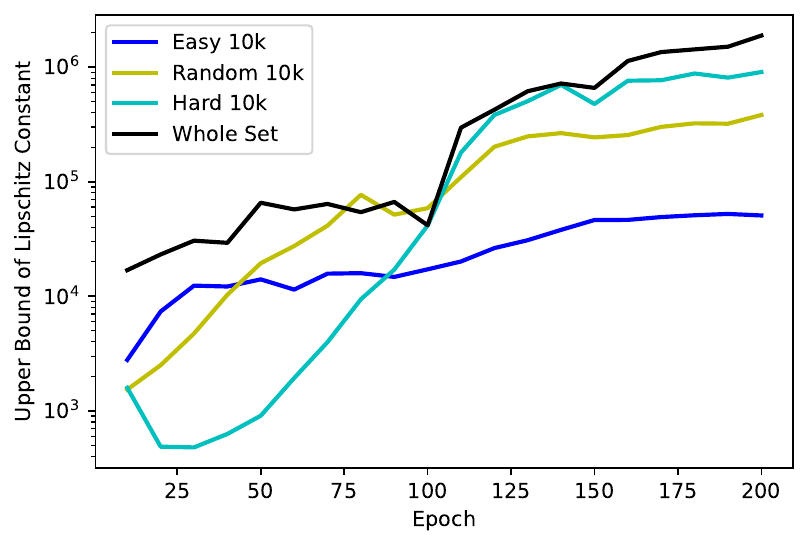}
\caption{Curves of the Lipschitz upper bound when the model is adversarially trained by the easiest, random, the hardest 10000 instances or the whole training set. The y-axis is in log-scale. \edit{Left: the adversarial budget is based on the $l_\infty$ norm with $\epsilon = 8 / 255$. Right: the adversarial budget is based on the $l_2$ norm with $\epsilon = 1$.}} \label{fig:lip_curve}
\end{figure}

Figure~\ref{fig:lip_curve} depicts the curves of the Lipschitz upper bound when the model is adversarially trained by the easiest, random, the hardest 10000 instances \edit{or the whole training set}.
The adversarial budget is based on the $l_\infty$ norm with $\epsilon = 8 / 255$.
We can clearly see that, as training progresses, the Lipschitz upper bound increases in all cases.
Furthermore, \edit{in the last phase of training, the Lipschitz estimation of the models adversarially trained on hard instances is bigger than the ones on easy instances.}
These results are consistent with Theorem~\ref{thm:nonlinear}.
\edit{In addition, when we conduct adversarial training on the whole training set, the Lipschitz bound is close to the one trained on the easy instances in the early phase of training, while it is close to the one trained on hard instances in the late phase.
This observations is consistent with what Corollary~\ref{coro:wholeset} indicates.}

\section{Case Study and Discussion} \label{sec:casestudy}

Our empirical and theoretical analyses indicate that fitting hard adversarial leads to adversarial overfitting.
\change{In this section, we first review existing approaches to mitigating overfitting in adversarial training.
Specifically, we show that they implicitly avoid fitting hard adversarial instances, which provides an explanation for their success.
We also show that the methods that encourage fitting hard adversarial instances fail to yield truly robust models.}

\change{We believe that our discovery is broadly applicable to different tasks aiming to achieve adversarial robustness against a norm-based adversarial budget.
In this regard, we study the cases of fast adversarial training and adversarial fine-tuning with additional training data.
Our results indicate that avoiding to fit hard adversarial instances also improves the performance in these cases.
More detailed discussions are deferred to Appendix~\ref{subsec:app_revisit}.} 


\subsection{\change{Standard Adversarial Training: A New Perspective on Existing Methods}} \label{subsec:standard}

Existing methods aiming to mitigate adversarial overfitting can be generally divided into two categories: those that use adaptive inputs, such as~\cite{balaji2019instance}, and those that rely on adaptive targets, such as~\cite{chen2021robust, huang2020self}.
We show below that both categories implicitly aim to prevent the model from fitting hard input-target pairs.

We use instance-wise adversarial training (IAT)~\citep{balaji2019instance} and self-adaptive training (SAT)~\citep{huang2020self} as examples of these two categories.
IAT uses an instance-adaptive adversarial budget during training.
It adaptively adjusts $\epsilon$, the size of the adversarial budget, for each training instance.
SAT uses self-supervised adaptive targets instead of the ground truth during training.
We run both algorithms using the settings in their original papers, except that we train the model for $200$ epochs for a consistent comparison.
The details are provided in Appendix~\ref{subsec:app_revisit}.

\begin{figure}[!ht]
\centering
\begin{subfigure}{0.3\columnwidth}
\includegraphics[width = \textwidth]{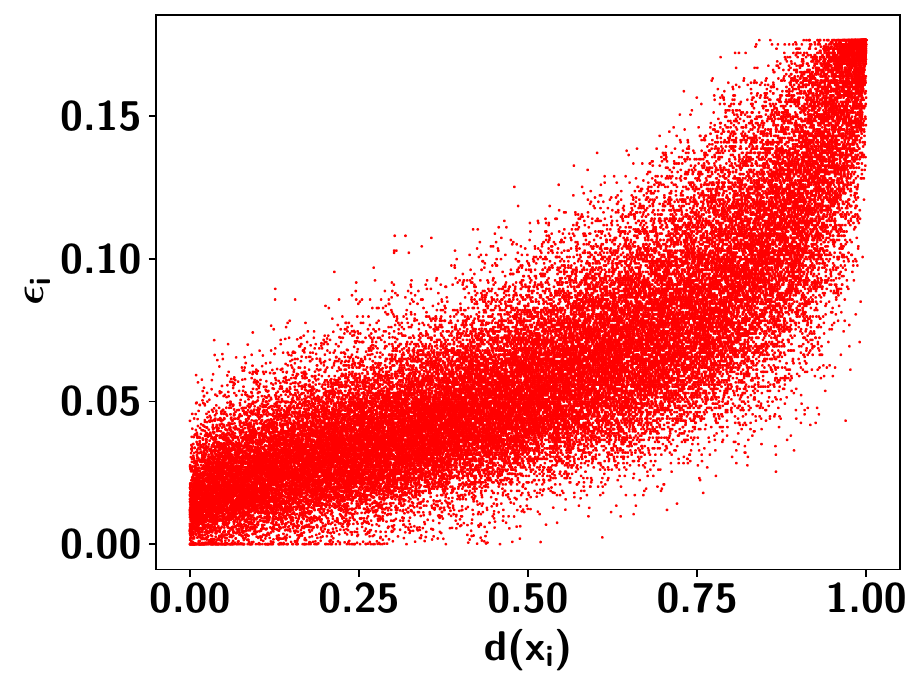}
\caption{\label{fig:instance_epsilon}}
\end{subfigure}
~~~~
\begin{subfigure}{0.3\columnwidth}
\includegraphics[width = \textwidth]{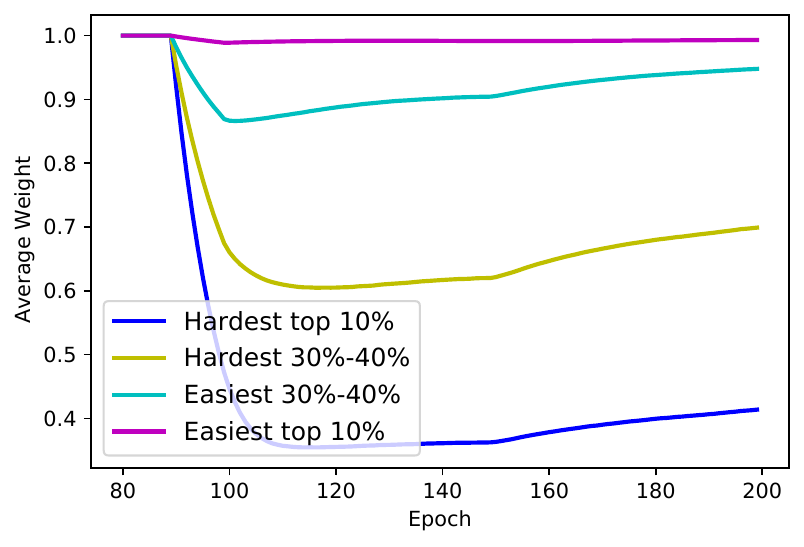}
\caption{\label{fig:movetarget_weight}}
\end{subfigure}
~~~~
\begin{subfigure}{0.3\columnwidth}
\includegraphics[width = \textwidth]{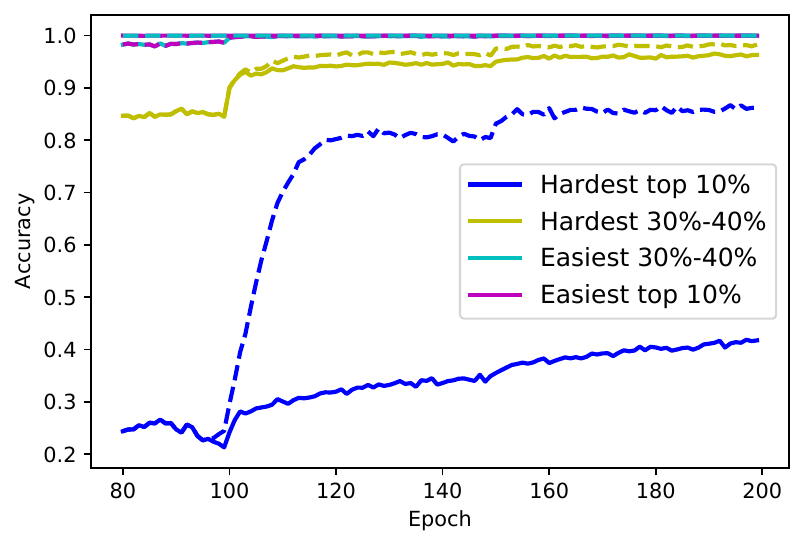}
\caption{\label{fig:movetarget_acc}}
\end{subfigure}
\caption{Results of our case study. The model is always an RN18 and the target adversarial budget's size $\epsilon = 8 /255$. (a) Relationship between instance difficulty $d(\vx_i)$ and its adversarial budget size in IAT for the CIFAR10 training set. (b) Average weights of different groups in the CIFAR10 training set during training in SAT. The warmup period is $90$ epochs, and SAT is enabled after that. (c) Training accuracy of different groups on the CIFAR10 training set during training in SAT. The solid lines and the dashed lines represent the accuracy on the ground truth and on the adaptive targets, respectively. The warmup period is $90$ epochs, and SAT is enabled after that.}
\end{figure}

Let us study how these algorithms adaptively use instances of different difficulty levels.
For IAT, we plot the relationship between the instance difficulty $d(\vx_i)$ and its adaptive adversarial budget's size $\epsilon_i$ in Figure~\ref{fig:instance_epsilon}, which shows a high correlation ($0.884$) between them.
Specifically, we find that the hard instances are assigned smaller adversarial budgets for training, which indicates that IAT prevents the model from fitting the hard adversarial instances.
For SAT, we show the average weights assigned to the instances in each group of $\{\gG_i\}_{i = 0}^9$ during training in Figure~\ref{fig:movetarget_weight}.
The hard instances are clearly assigned much smaller weights to calculate the loss, which indicates that they are downplayed during training.
We also provide the average accuracy of each group during training in Figure~\ref{fig:movetarget_acc}, given both the ground truth or the adaptive target.\footnote{For the adaptive target $\vt$, the prediction $\vo$ is considered correct if and only if $\argmax_i \vt_i = \argmax_i \vo_i$.}
We observe that the hard instances have much higher accuracy on their adaptive targets compared with the ground truth, while such a difference is much smaller for the easy instances.
Our results thus indicate that the adaptive targets used by SAT are much easier to fit, which avoids having to directly fit the hard adversarial input-target pairs.

In addition to IAT and SAT, other methods have introduced regularization terms to mitigate adversarial overfitting, such as~\cite{zhang2019theoretically} and~\cite{chen2021robust}.
These regularization terms calculate the distance between the adversarial output logits and their anchor points.
The anchor points are the adaptive targets, and can be the clean output logits in~\cite{zhang2019theoretically} or a teacher network's outputs in~\cite{chen2021robust}.
The regularizers used in these methods encourage the adversarial output logits to be closer to the anchor points other than to the ground truth for the hard instances.
In other words, these methods also use adaptive targets to avoid fitting the hard input-target pairs.

In contrast to the methods above, \cite{zhang2021geometryaware} proposed an instance-adaptive reweighting strategy which assigns larger weights to the training instances that PGD breaks in fewer iterations.
In other words, this approach assigns larger weights to the hard adversarial instances, which contrasts with what our analysis revealed.
As a matter of fact, this method was recently shown to be vulnerable to adaptive attacks~\citep{hitaj2021evaluating}. 

\subsection{Alternative Training Scenarios} \label{subsec:more_casestudy}

\change{We believe that our findings can be applied to improve the generalization performance of robust models in various situations.}
In this regard, we conduct preliminary analyses on two examples: fast adversarial training and fine-tuning a pre-trained model using additional data.
\change{In these examples, we show consistent observations with standard adversarial training.}
\revision{Our focus in this section is to showcase the general applicability of our findings rather than proposing entirely new algorithms.
Our results below demonstrate that avoiding fitting hard adversarial instances can consistently mitigate overfitting and improve models' robustness in various scenarios.}

\subsubsection{\change{Fast Adversarial Training}}

Adversarial training in~\cite{madry2017towards} introduces a significant computational overhead.
Thus it is desirable to accelerate this method.
\edit{This section studies how adaptive training based on the instances' difficulty mitigates overfitting and improves fast adversarial training.}
\edit{Specfically,} our experiments in this section are based on adversarial training with transferable adversarial examples (ATTA in~\cite{zheng2020efficient}), which stores the adversarial perturbation for each training instance as an initial point for the next epoch.

First, we use a reweighting scheme to assign lower weights to hard instances when calculating the loss objective: each training instance is assigned a weight equal to the adversarial output probability of the true label.
Then this weight is normalized to ensure that the weights in a mini-batch sum to $1$.
Note that our reweighting scheme is based on the adversarial output instead of the clean output, because the adversarial output probability will also be used to calculate the loss objective.
As a result, the computational overhead of the reweighting scheme is negligible.

\revision{In addition to reweighting, we adapt SAT~\citep{huang2020self} to fast adversarial training and quantitatively study how adaptive targets for hard adversarial training instances mitigate adversarial overfitting.}
For each training instance $(\vx, y)$, we maintain an adaptive moving average target $\widetilde{\vt}$. $\widetilde{\vt}$ is updated in an exponential averaging manner for each epoch: $\widetilde{\vt} \leftarrow \rho\widetilde{\vt} + (1 - \rho) \vo'$ where $\rho$ is the momentum factor and $\vo'$ is the logit of the adversarial input $\vx'$.
\edit{Like the reweighting scheme, compared with~\cite{huang2020self},} we use the adversarial output $\vo'$ instead of the clean output $\vo$ to avoid computational overhead.
The final adaptive target we use is $\vt = \beta \mathbf{1}_y + (1 - \beta)\widetilde{\vt}$ and thus the loss objective is $\gL_\vw(\vx', \vt)$.
The factor $\beta$ controls how ``adaptive'' our target is: $\beta = 0$ yields a fully adaptive moving average target $\widetilde{\vt}$ and $\beta = 1$ yields a one-hot target $\mathbf{1}_y$.
We provide the pseudocode as Algorithm~\ref{alg:fast}.

Our experiment is on CIFAR10 and use $l_\infty$ norm based adversarial budget with $\epsilon = 8 / 255$, the standard setting where most fast adversarial training algorithms are benchmarked~\cite{croce2020robustbench}.
\edit{Unless specified, we use the same settings as in~\cite{zheng2020efficient}.}
we train the model for $38$ epochs, the learning rate is $0.1$ on the first $30$ epochs, it decays to $0.01$ in the next $6$ epochs and further decays to $0.001$ in the last $2$ epochs.
We evaluate the model's robust accuracy on the test set by AutoAttack~\cite{croce2020reliable}, the popular and reliable attack for evaluation.
More hyper-parameter details are deferred to Appendix~\ref{subsec:settings_casestudy}

\begin{algorithm}[!ht]
\begin{algorithmic}
\STATE \textbf{Input:} training data $\gD$, model $f$, batch size $B$, PGD step size $\alpha$, adversarial budget $\gS^{(p)}(\epsilon)$, coefficient $\rho$, $\beta$.
\FOR {Sample a mini-batch $\{\vx_i, y_i\}_{i = 1}^B \sim \gD$}
    \STATE $\forall i$, obtain the initial perturbation $\Delta_i$ as in~\cite{zheng2020efficient}.
    \STATE $\forall i$, one step PGD update: $\Delta_i \leftarrow \Pi_{\gS^{(p)}(\epsilon)}\left[\Delta_i + \alpha sign(\triangledown_{\Delta_i} \gL_\theta(\vx_i + \Delta_i, y_i)\right])$.
    \STATE $\forall i$, update the cached adversarial perturbation $\Delta_i$ as in~\cite{zheng2020efficient}.
    \IF {use reweight}
        \STATE $\forall i$, weight $w_i = \mathrm{softmax}[f(\vx_i + \Delta_i)]_{y_i}$
    \ELSE
        \STATE $\forall i$, weight $w_i = 1$
    \ENDIF
    \STATE $\forall i$, query the adaptive target $\tilde{\vt}_i$ and update: $\tilde{\vt}_i \leftarrow \rho \tilde{\vt}_i + (1 - \rho) softmax[f(\vx_i + \Delta_i)]$.
    \STATE $\forall i$, the final adaptive target $\vt_i = \beta \mathbf{1}_{y_i} + (1 - \beta) \tilde{\vt_i}$
    \STATE Calculate the loss $\frac{1}{\sum_i^B w_i} \sum_i^B w_i \gL_\theta(\vx_i + \Delta_i, \vt_i)$ and update the parameters.
\ENDFOR
\end{algorithmic}
\caption{One epoch of the accelerated adversarial training.} \label{alg:fast}
\end{algorithm}

\begin{table}[!ht]
\centering
\begin{tabular}{p{4.5cm}p{1.4cm}p{1.4cm}<{\centering}p{1.5cm}<{\centering}p{3.2cm}<{\centering}}
\Xhline{4\arrayrulewidth}
Method & Model & Epochs & Complexity & AutoAttack(\%) \\
\Xhline{4\arrayrulewidth}
\cite{shafahi2019adversarial} & WRN34  & 200 & 2 & 41.17 \\
\cite{wong2020fast}           & RN18   & 15  & 4 & 43.21 \\
\cite{zheng2020efficient}     & WRN34  & 38  & 4 & 44.48 \\
\cite{zhang2019you}           & WRN34  & 105 & 3 & 44.83 \\
\cite{chen2021efficient}      & WRN34  & 100 & 7 & 51.12 \\
\hdashline
Reweighting (Ours)            & WRN34  & 38  & 4 & 46.15 \\
Adaptive Target (Ours)        & WRN34  & 38  & 4 & 51.17 \\
\Xhline{4\arrayrulewidth}
\end{tabular}
\caption{Comparison between different accelerated adversarial training methods in robust test accuracy against AutoAttack (AA). The baseline results are from RobustBench. \textit{Complexity} shows the total number of forward passes and backward passes in one mini-batch update.} \label{tbl:fast}
\end{table}

The results are provided in Table~\ref{tbl:fast}, where the results of the baseline methods are taken from RobustBench~\cite{croce2020robustbench}.
We also report the number of epochs and the number of forward and backward passes in a mini-batch update of each method.
The product of these two values indicates the training complexity.
We can clearly see that both reweighting and adaptive targets improve the performance on top of ATTA~\cite{zheng2020efficient}.
Note that our method based on adaptive targets achieve the best performance while needing only $1/4$ of the training time of~\cite{chen2021efficient}, the strongest baseline.
\cite{wong2020fast} is the only baseline consuming less training time than ours, but its performance is much worse than ours; it suffers from catastrophic overfitting when using a WRN34 model.

We also conduct ablation study in the context of fast adversarial training.
In Figure~\ref{fig:fast_adv_ablation}, we plot the learning curves for different values of $\beta$ in Algorithm~\ref{alg:fast}, we also compare the learning curves of ATTA with and without reweighting.
Lower the value of $\beta$ is, more weights assigned to the adaptive part of the target: $\beta = 0$ means we directly utilize the moving average target as the final target, $\beta = 1$ means we use the one-hot groundtruth label.
In the left part of Figure~\ref{fig:fast_adv_ablation}, the generalization gap decreases with the decrease in $\beta$.
That is to say, the adaptive target can indeed improve the generalization performance.
In addition, the right part of Figure~\ref{fig:fast_adv_ablation} confirm that the reweighting scheme can prevent adversarial overfitting and decrease the generalization gap.

\begin{figure}[!ht]
\centering
\includegraphics[width = 0.42\textwidth]{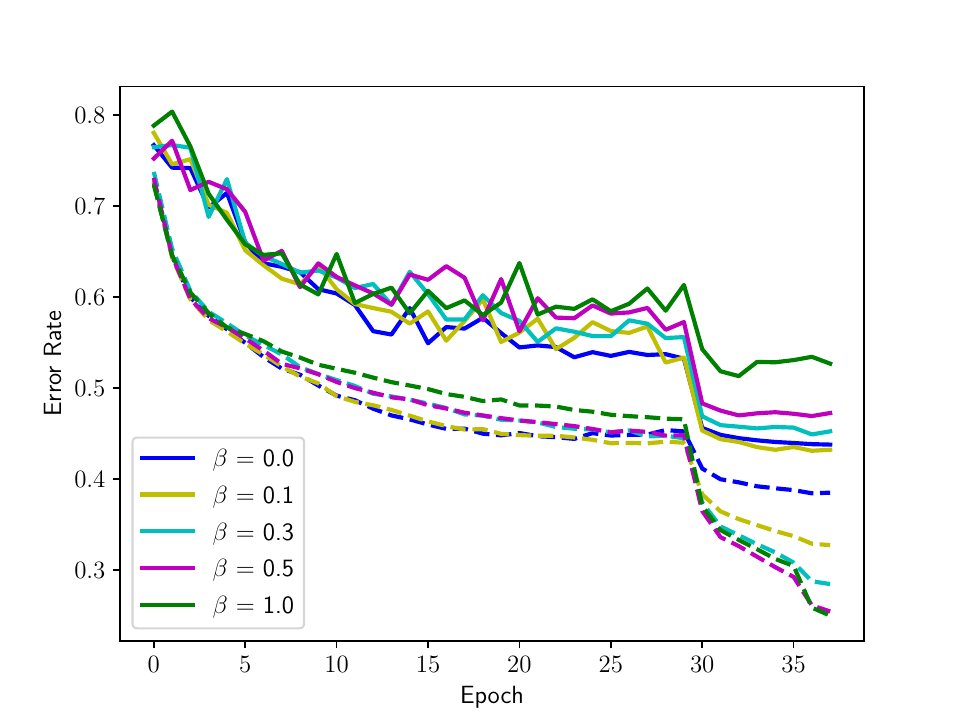}
~~~~~~
\includegraphics[width = 0.42\textwidth]{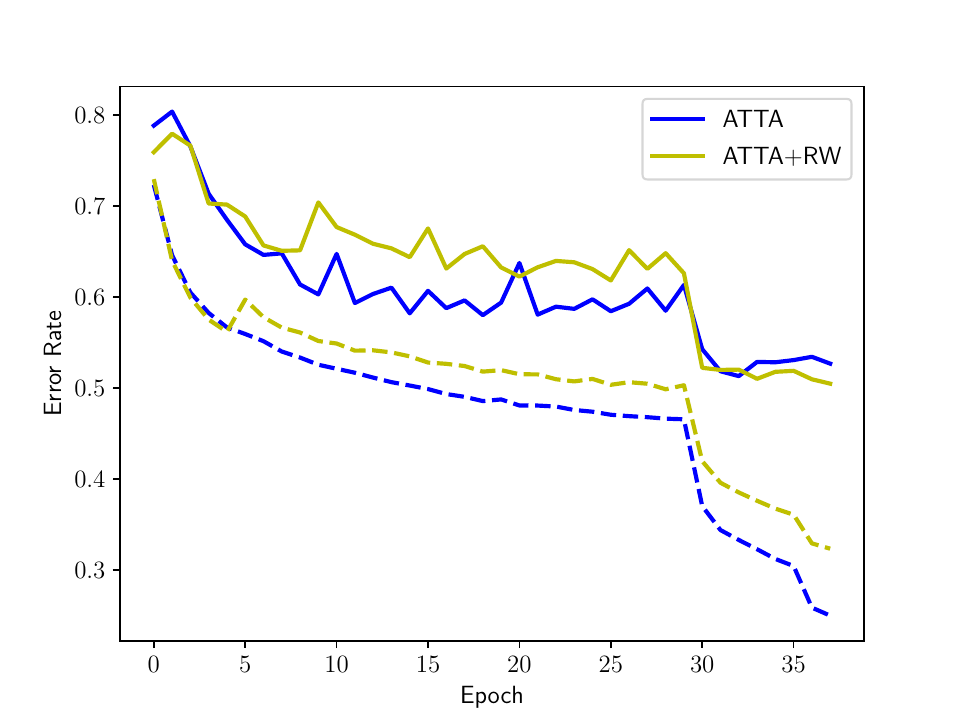}
\caption{The learning curves of Algorithm~\ref{alg:fast} when we use different values of $\beta$ (left) or compare the performance with and without reweighting (right). The solid curve and the dashed curve represent the robust test error and the robust training error, respectively.}
\label{fig:fast_adv_ablation}
\end{figure}


\subsubsection{\change{Adversarial Fine-tuning with Additional Data}}
\revision{In this section, we study fine-tuning an adversarially pretrained model using additional training data.
We observe that adversarial overfitting occurs when using a small learning rate in Section~\ref{sec:overfit}.
Since we also use a small learning rate to conduct adversarial fine-tuning with additional data, it is important to address the adversarial overfitting issue in this context.}
While additional training data was shown to be beneficial in~\cite{alayrac2019labels,carmon2019unlabeled}, we demonstrate that letting the model adaptively fit the easy and hard instances of the additional data further improve the performance.

We conduct experiments on both CIFAR10 and SVHN, using WRN34 and RN18 models, respectively.
The model is fine-tuned for either $1$ epoch or $5$ epochs, which means that each additional training instance is used either $5$ times or only once.
This is because we observed the performance of vanilla adversarial training to start decaying after $5$ epochs.
As such, methods requiring many epochs such as~\cite{balaji2019instance} and~\cite{huang2020self} are not applicable here.
More hyper-parameter details are deferred to Appendix~\ref{subsec:settings_casestudy}.

Our first technique, reweighting, is the same as in the previous section.
In addition to reweighting, we can also add a KL regularization term measuring the KL divergence between the output probability of the clean instance and of the adversarial instance.
The KL term encourages the adversarial output to be close to the clean one.
In other words, the clean output probability serves as the adaptive target.
For hard instances, the clean and adversarial inputs are usually both misclassified.
Therefore, the clean outputs of these instances constitute simpler targets compared with the ground-truth labels.
Ultimately, the loss objective of a mini-batch $\{\vx_i\}_{i = 1}^B$ used for fine-tuning is expressed as $\gL_{FT}(\{\vx_i\}_{i = 1}^B) = \sum_{i = 1}^B w_i \left[ \gL_\vw(\vx'_i) + \lambda KL(\vo_i || \vo'_i) \right]$
where $w_i$ is the adaptive weight when we use re-weighting, or $1 / B$ otherwise.
$\lambda$ is $6$ when using the regularization term and $0$ otherwise.

We use reweighting and KL regularization to fine-tune the model.
Results in Table~\ref{tbl:ft} clearly show that both techniques benefit the performance of the finetuned model.
This shows that avoiding fitting hard adversarial examples helps to improve the generalization performance in adversarial fine-tuning with additional training data.

\begin{table}[!ht]
\centering
\begin{tabular}{p{1.5cm}<{\centering}p{2.1cm}p{2.4cm}<{\centering}:p{1.5cm}<{\centering}p{2.1cm}p{2.4cm}<{\centering}}
\Xhline{4\arrayrulewidth}
Duration & Method & AutoAttack(\%) & Duration & Method & AutoAttack(\%) \\
\Xhline{4\arrayrulewidth}
\multicolumn{3}{c:}{\textbf{WRN34 on CIFAR10, $\epsilon = 8 /255$}} & \multicolumn{3}{c}{\textbf{RN18 on SVHN, $\epsilon = 0.02$}} \\
\hline
\multicolumn{2}{l}{No Fine Tuning} & 52.01 & \multicolumn{2}{l}{No Fine Tuning} & 67.77 \\
\hdashline
\multirow{4}{*}{1 Epoch} & Vanilla AT & 54.11 & \multirow{4}{*}{1 Epoch} & Vanilla AT & 70.81 \\
                         & RW         & 54.69 &                          & RW         & 70.83 \\
                         & KL         & 54.73 &                          & KL         & 72.29 \\
                         & RW + KL    & 54.69 &                          & RW + KL    & 72.53 \\
\hdashline
\multirow{4}{*}{5 Epoch} & Vanilla AT & 55.49 & \multirow{4}{*}{5 Epoch} & Vanilla AT & 72.18 \\
                         & RW         & 56.41 &                          & RW         & 72.72 \\
                         & KL         & 56.55 &                          & KL         & 73.17 \\
                         & RW + KL    & 56.99 &                          & RW + KL    & 73.35 \\
\Xhline{4\arrayrulewidth}
\end{tabular}
\caption{Robust accuracy of fine-tuned models against AutoAttack(AA). We conduct ablation study on both reweighting (RW) and KL regularization (KL).} \label{tbl:ft}
\end{table}

\section{Conclusion}
We have investigated \textit{adversarial overfitting} from the perspective of training instances' difficulty.
By introducing a quantitative metric to measure the instance difficulty, we have shown that a model's generalization performance under adversarial attacks degrades during the later phase of training as the model fits the hard adversarial instances.
We have conducted theoretical analyses on both linear and nonlinear models.
On an over-parameterized logistic regression model, we have shown that training on harder adversarial instances leads to poorer generalization performance.
We have also proven that the performance of adversarial training is more sensitive to hard instances than vanilla training.
On general nonlinear models, we have shown that the lower bound of a well-trained model's Lipschitz constant increases when trained with more difficult instances.
Finally, we have shown that existing approaches to mitigating adversarial overfitting implicitly avoid fitting hard adversarial instances.
We believe that our findings shed some light on adversarial training, and will allow the community to design new algorithms and improve robustness in diverse applications.

\section*{Acknowledgment}

Part of this work is supported by National Natural Science Foundation of China (NSFC Project No. 62306250) and CityU APRC Project (Project No. 9610614).




\newpage

\appendix

\section{Notation} \label{sec:notation}

\begin{table}[!ht]
\small
\centering
\begin{tabular}{p{0.07\textwidth}|p{0.3\textwidth}|p{0.58\textwidth}}
\Xhline{4\arrayrulewidth}
$\gA$ & Section~\ref{sec:hardeasy} & Perturbation method. \\
$b$ & Section~\ref{sec:thm_nonlinear} & The number of parameters in a general nonlinear model. \\
$c$ & Assumption~\ref{asp:iso}, Section~\ref{sec:thm_nonlinear} & The coefficient in isoperimetry. \\
$C$ & Section~\ref{sec:thm_nonlinear} & The mean squared error on the adversarial training set. \\
$d$ & Equation~\ref{eq:difficulty}, Section~\ref{sec:hardeasy} &  The function representing the difficulty metric. \\
$\gD$ & Section~\ref{sec:hardeasy} & The data set. \\
$f_\vw$ & Section~\ref{sec:intro} & The model parameterized by $\vw$. \\
$\gF$ & Theorem~\ref{thm:nonlinear}, Section~\ref{sec:thm_nonlinear} & The function space of the model. \\
$\gG$ & Section~\ref{subsec:hardoverfit} & Groups of the training set divided by instance difficulty. \\
$h$ & Definition~\ref{def:h}, Section~\ref{sec:thm_nonlinear} & The bandwidth of the model's output range. \\
$J$ & Theorem~\ref{thm:nonlinear}, Section~\ref{sec:thm_nonlinear} & The Lipschitz constant of $f_\vw$ w.r.t $\vw$. \\
$K$ & Section~\ref{sec:thm} & The number of components in the data distribution. \\
$l$ & Section~\ref{sec:thm} & The component index where the training data is sampled. \\
$L$ & Assumption~\ref{asp:iso}, Section~\ref{sec:thm_nonlinear} & The Lipschitz constant of $f_\vw$ w.r.t the input. \\
$\gL$ & Section~\ref{sec:intro} & The loss function. \\
$m$ & Section~\ref{sec:thm} & Dimension of the input data. \\
$M$ & Section~\ref{sec:hardeasy} & The number of total training epochs. \\
$n$ & Section~\ref{sec:thm} & The number of training instances. \\
$\vo$, $\vo'$ & Section~\ref{sec:casestudy} & Model's output of the clean and the adversarial input. \\
$p$ & Section~\ref{sec:intro} & Shape of the adversarial budget. \\
$p_k$ & Section~\ref{sec:linear} & The probability of $k$-th component in the GMM model.\\
$r$ & Equation~\ref{eq:gmm}, Section~\ref{sec:linear} & The coefficient in the GMM model. \\
$\gR$ & Theorem~\ref{thm:main}, Section~\ref{sec:linear} & The robust test error. \\
$\vt$, $\widetilde{\vt}$ & Section~\ref{subsec:more_casestudy} & The adaptive target and the moving average target. \\
$\vw$ & Section~\ref{sec:thm} & Model parameters. \\
$W$ & Theorem~\ref{thm:nonlinear}, Section~\ref{sec:thm_nonlinear} & The diameter upper bound of the parameter space. \\
$\gW$ & Theorem~\ref{thm:nonlinear}, Section~\ref{sec:thm_nonlinear} & The space of model parameters. \\
$\vx, \vx', \rmX$ & Section~\ref{sec:intro} \& Section~\ref{sec:thm} & Clean input, adversarial input and its matrix form. \\
$y$, $\vy$ & Section~\ref{sec:intro} \& Section~\ref{sec:thm} & Label and its vector form. \\
$\alpha$ & Algorithm~\ref{alg:fast} & The step size of the adversarial attacks. \\
$\beta$ & Section~\ref{subsec:more_casestudy} & The coefficient controlling how adaptive the target is. \\
$\gamma$ & Theorem~\ref{thm:nonlinear}, Section~\ref{sec:thm_nonlinear} & The non-negative variable introduced in Theorem~\ref{thm:nonlinear}. \\
$\delta$ & Theorem~\ref{thm:nonlinear}, Section~\ref{sec:thm_nonlinear} & The probability introduced in Theorem~\ref{thm:nonlinear}. \\
$\epsilon$ & Section~\ref{sec:intro} & The size of the adversarial budget. \\
$\veta$ & Equation~\ref{eq:gmm}, Section~\ref{sec:linear} & The direction of the mean of each GMM's component. \\
$\rho$ & Section~\ref{subsec:more_casestudy} & The momentum calculating the moving average target. \\
$\mu_l$, $\mu_l$ & Assumption~\ref{asp:iso}, Section~\ref{sec:thm_nonlinear} & Data distribution and its $l$-th component. \\
$\sigma$ & Assumption~\ref{asp:iso}, Section~\ref{sec:thm_nonlinear} & The conditional variance of the data distribution. \\

\hline
\Xhline{4\arrayrulewidth}
\end{tabular}
\caption{The notation in this paper. In addition to what they represent, we provide the section of their definition or first appearance.}
\label{tbl:notation}
\end{table}


\section{Proofs in Theoretical Analysis}

\subsection{Proof of Theorem~\ref{thm:converge}} \label{sec:proof_converge}

Similar to~\cite{soudry2018implicit}, we can assume all instances are positive without the loss of generality, this is because we can always redefine $y_i\vx_i$ as the input.
In this regard, the loss to optimize in a logistic regression model under the adversarial budget $\gS^{(2)}(\epsilon)$ is:

\begin{equation}
\begin{aligned}
\gL_\vw(\rmX) = \sum_{i = 1}^n l(\vw^T \vx_i - \epsilon \|\vw\|)
\end{aligned}
\end{equation}

Here $l(\cdot)$ is the logistic function: $l(x) = \frac{1}{1 + e^{-x}}$.
We use $\rmX \in \sR^{n \times m}$ to represent the training set as said in Section~\ref{sec:thm}, then the loss function $\gL(\vw)$ is $\|\rmX\|^2$-smooth, where $\|\rmX\|^2$ is the maximal singular value of $\rmX$.
Since function $\gL_\vw$ is convex on $\vw$, so gradient descent of step size smaller than $2\|\rmX\|^{-2}$ will asymptotically converge to the global infimum of the function $\gL_\vw$ on $\vw$.

Before proving Theorem~\ref{thm:converge}, we first introduce the following lemma:

\begin{lemma} \label{lem:equiv}
Consider the max-margin vector $\widehat{\vw}$ of the vanilla case defined in Equation~(\ref{eq:max_margin}), we then introduce the max margin vector $\widehat{\vw'}$ defined under the adversarial attack of budget $\gS^{(2)}(\epsilon)$ as follows:
\begin{equation}
\begin{aligned}
\widehat{\vw'} = \argmin_{\vw} \|\vw\| \ \ &s.t. \ \forall i \in \{1, 2, ..., n\}, \ \vw^T \vx_i - \epsilon \|\vw\| \geq 1
\end{aligned} \label{eq:adv_max_margin}
\end{equation}
Then we have $\widehat{\vw'}$ is collinear with $\widehat{\vw}$, i.e., $\frac{\widehat{\vw'}}{\|\widehat{\vw'}\|} = \frac{\widehat{\vw}}{\|\widehat{\vw}\|}$
\end{lemma}

\begin{proof}
We show that $\widehat{\vw} = \frac{1}{1 + \epsilon \|\widehat{\vw'}\|} \widehat{\vw'}$ and prove it by contraction.

Let's assume $\exists \vv,\ s.t.\ \|\vv\| < \frac{\|\widehat{\vw'}\|}{1 + \epsilon \|\widehat{\vw'}\|} \mathrm{and} \ \forall i \in \{1, 2, ..., n\}, \ \vv^T \vx_i \geq 1$, then we can consider $\vv' = (1 + \|\widehat{\vw'}\|)\vv$.
The $l_2$ norm of $\vv'$ is smaller than that of $\widehat{\vw'}$, and we have
\begin{equation}
\begin{aligned}
\forall i \in \{1, 2, ..., n\}, \vv'^T\vx_i - \epsilon \|\vv'\| = (1 + \epsilon\|\widehat{\vw'}\|)\vv^T\vx_i - \epsilon\|\vv'\| > (1 + \epsilon \|\widehat{\vw'}\|) - \epsilon \|\widehat{\vw'}\| = 1
\end{aligned} \label{eq:contraction}
\end{equation}

Inequality~\ref{eq:contraction} shows we can construct a vector $\vv'$ whose $l_2$ norm is smaller than $\widehat{\vw'}$ and satisfying the condition~(\ref{eq:adv_max_margin}), this contracts with the optimality of $\widehat{\vw'}$.
Therefore, there is no solution of condition~(\ref{eq:max_margin}) whose norm is smaller than $\frac{\|\widehat{\vw'}\|}{1 + \epsilon \|\widehat{\vw'}\|}$.

On the other hand, $\frac{1}{1 + \epsilon \|\widehat{\vw'}\|} \widehat{\vw'}$ satisfies the condition~(\ref{eq:max_margin}) and its $l_2$ norm is $\frac{\|\widehat{\vw'}\|}{1 + \epsilon \|\widehat{\vw'}\|}$.
As a result, we have $\widehat{\vw} = \frac{1}{1 + \epsilon \|\widehat{\vw'}\|} \widehat{\vw'}$.
That means $\widehat{\vw}$ and $\widehat{\vw'}$ are collinear.
\end{proof}

With Lemma~\ref{lem:equiv}, Theorem~\ref{thm:converge} is more straightforward, whose proof is shown below.
Regarding the convergence analysis of the logistic regression model in non-adversarial cases, we encourage the readers to find more details in~\cite{ji2019implicit, soudry2018implicit}.

\begin{proof}
Theorem 1 in~\cite{ji2019implicit} and Theorem 3 in~\cite{soudry2018implicit} proves the convergence of the direction of the logistic regression parameters in different cases.
In this regard, we can let $\vw_\infty = \lim_{u \to \infty} \frac{\vw(u)}{\|\vw(u)\|}$.
That is to say, for sufficiently large $u$, the direction of the parameter $\vw(u)$ can be considered fixed.
As a result, the adversarial perturbations of each data instance $\vx_i$ is fixed, i.e., $\epsilon \vw_\infty$.

We can then apply the conclusion of Theorem 3 in~\cite{soudry2018implicit} here, the only difference is the data points are $\{\vx_i - \epsilon \vw_\infty\}_{i = 1}^n$.
Therefore, the parameter $\vw(u)$ will converge to the $l_2$ max margin of the dataset $\{\vx_i -\epsilon \vw_\infty\}_{i = 1}^n$.
When $t \to \infty$, we have $\vw(u)^T (\vx_i - \epsilon \vw_\infty) = \vw(u)^T\vx_i - \epsilon \|\vw(u)\|$.
This is exactly the adversarial max margin condition in (\ref{eq:adv_max_margin}).
Based on Lemma~\ref{lem:equiv}, we have $\lim_{u \to \infty} \frac{\vw(u)}{\|\vw(u)\|} = \frac{\widehat{\vw'}}{\|\widehat{\vw'}\|} = \frac{\widehat{\vw}}{\|\widehat{\vw}\|}$
\end{proof}

\subsection{Proof of Theorem~\ref{thm:main}} \label{sec:proof_main}




Given the parameter $\vw$ of the logistic regression model, we can first calculate the robust error for the $k$-th component of the GMM model defined in~(\ref{eq:gmm}).


\begin{lemma} \label{lemma:acc}
The 0-1 classification error of a linear classifier $\vw$ under the adversarial attack of the budget $\gS^{(2)}(\epsilon)$ for the $k$-th component of the GMM model defined in~(\ref{eq:gmm}) is:
\begin{equation}
\begin{aligned}
\widehat{\gR}_k(\epsilon) = \Phi(\frac{r_k \vw^T \veta}{\|\vw\|} - \epsilon)
\end{aligned}
\end{equation}
where $\Phi(x) = \sP(Z > x), Z \sim \gN(0, 1)$.
\end{lemma}

\begin{proof}
For a random drawn data instance $(\vx, y)$, the adversarial perturbation is $-y \epsilon \frac{\vw}{\|\vw\|}$.
Let's decompose $\vx$ as $r_k y\veta + \vz$, where $\vz \sim \gN(0, \rmI)$.
Then, we have
\begin{equation}
\begin{aligned}
\widehat{\gR}_k(\epsilon) &= \sP(y\vw^T(\vx - y\epsilon\frac{\vw}{\|\vw\|}) < 0) = \sP(y\vw^T(r_k y\veta + \vz - y\epsilon\frac{\vw}{\|\vw\|}) < 0) \\
         &= \sP(-y\vw^T\vz > r_k \vw^T\veta - \epsilon\|\vw\|)
\end{aligned}
\end{equation}
Since $\vz \sim \gN(0, \rmI)$, we have $-y\vw^T\vz \sim \gN(0, (-y\vw^T)^T(-y\vw^T)) = \gN(0, \vw^T\vw)$.
Furthermore $\frac{-y\vw^T\vz}{\|\vw\|} \sim \gN(0, 1)$, and we can further simplify $\widehat{\gR}_k(\epsilon)$ as follows:

\begin{equation}
\begin{aligned}
\widehat{\gR}_k(\epsilon) = \sP(\frac{-y\vw^T\vz}{\|\vw\|} > \frac{r_k \vw^T\veta}{\|\vw\|} - \epsilon) = \Phi(\frac{r_k \vw^T \veta}{\|\vw\|} - \epsilon)
\end{aligned}
\end{equation}
\end{proof}

With Lemma~\ref{lemma:acc}, we can straightforwardly calculate the robust error for all components of the GMM model defined in~(\ref{eq:gmm}):

\begin{equation}
\begin{aligned}
\widehat{\gR}(\epsilon) = \sum_{k = 1}^K p_k \Phi(\frac{r_k \vw^T \veta}{\|\vw\|} - \epsilon)
\end{aligned} \label{eq:testacc}
\end{equation}

On the other hand, Theorem~\ref{thm:converge} indicates the parameter $\vw$ will converge to the $l_2$ max margin.
However, for arbitrary training set, we do not have the closed form of $\vw$, which is a barrier for the further analysis.
Nevertheless, results from~\cite{wang2020benign} indicates in the over-parameterization regime, the parameter $\vw$ will converge to min-norm interpolation of the data with high probability.

\begin{lemma}{(Directly from Theorem 1 in~\cite{wang2020benign})} \label{lemma:min_norm}
Assume $n$ training instances drawn from the $l$-th mode of the described distribution in~(\ref{eq:gmm}) and each of them is a $m$-dimensional vector. If $\frac{m}{n\log n}$ is sufficiently large\footnote{Specifically, $m$ and $n$ need to satisfy $m > 10 n \log n + n -1$ and $m > C n r_l \sqrt{\log 2 n}\|\veta\|$. The constant $C$ is derived in the proof of Theorem 1 in~\cite{wang2020benign}.}, then the $l_2$ max margin vector in Equation (\ref{eq:max_margin}) will be the same as the solution of the min-norm interpolation described below with probability at least $(1 - O(\frac{1}{n}))$.
\begin{equation}
\begin{aligned}
\bar{\vw} = \argmin_{\vw} \|\vw\| \ \ s.t. \ \forall i \in \{1, 2, ..., n\}, \ y_i = \vw^T \vx_i
\end{aligned} \label{eq:min_norm}
\end{equation}
\end{lemma}
Since the min-norm interpolation has a closed solution $\bar{\vw} = \rmX^T (\rmX \rmX^T)^{-1}\vy$, Lemma~\ref{lemma:min_norm} will greatly facilitate the calculation of $\sR(\vw)$ in Theorem~\ref{thm:main}.
To simplify the notation, we first define the following variables.

\begin{equation}
\begin{aligned}
\rmU = \rmQ\rmQ^T,\ \vd = \rmQ\veta,\ s = \vy^T\rmU^{-1}\vy,\ t = \vd\rmU^{-1}\vd,\ v = \vy^T\rmU^{-1}\vd 
\end{aligned} \label{eq:def}
\end{equation}
The proof of Theorem~\ref{thm:main} is then presented below.



\begin{proof}
Based on (\ref{eq:testacc}), the key is to simplify the term $\frac{\vw^T \veta}{\|\vw\|}$, let's denote it by $A$, then we have:
\begin{equation}
\begin{aligned}
A^2 = \frac{\veta^T \vw \vw^T \veta}{\vw^T \vw} = \frac{(\vy^T(\rmX\rmX^T)^{-1}\rmX\veta)^2}{\vy^T(\rmX\rmX^T)^{-1}\vy}
\end{aligned} \label{eq:r2}
\end{equation}

The key challenge here is to calculate the term $(\rmX\rmX^T)^{-1}$ where $\rmX = r_l \vy\veta^T + Q$.
Here we utilize Lemma 3 of~\cite{wang2020benign} and Woodbury identity~\cite{horn2012matrix}, we have:



\begin{equation}
\begin{aligned}
\vy^T (\rmX\rmX)^{-1} = \vy^T\rmU^{-1} - \frac{(r_l^2 s\|\veta\|^2 + r_l^2v^2 + r_l v -r_l^2 st)\vy^T +r_l s\vd^T}{r_l^2s(\|\veta\|^2 - t) + (r_l v + 1)^2}\rmU^{-1}
\end{aligned} \label{eq:inverse}
\end{equation}

Here, $s$, $t$, $v$, $\rmU$ and $\vd$ are defined in Equation (\ref{eq:def}).
The scalar divisor comes from the matrix inverse calculation.
Base of Equation (\ref{eq:inverse}), we can then calculate $\vy^T(\rmX\rmX^T)^{-1}\vy$ and $\vy^T(\rmX\rmX^T)^{-1}\rmX\veta$.



\begin{equation}
\begin{aligned}
\vy^T(\rmX\rmX^T)^{-1}\vy &= s - \frac{(r_l^2 s \|\veta\|^2 + r_l^2 v^2 + r_l v - r_l^2st)s + r_l sv}{r_l^2 s(\|\veta\|^2 - t) + (r_l v + 1)^2} \\
                            &= \frac{s}{r_l^2 s(\|\veta\|^2 - t) + (r_l v + 1)^2}
\end{aligned} \label{eq:yy}
\end{equation}

\begin{equation}
\begin{aligned}
\vy^T(\rmX\rmX^T)^{-1}\rmX\veta &= \vy^T(\rmX\rmX^T)^{-1}(r_l\vy\veta^T + Q)\veta \\
                                &= r_l \|\veta\|^2\vy^T(\rmX\rmX^T)^{-1}\vy + \vy^T(\rmX\rmX^T)^{-1}\vd \\
                                &= \frac{r_l s(\|\veta\|^2 - t) + r_l v^2 + v}{r_l^2 s(\|\veta\|^2 - t) + (r_l v + 1)^2}
\end{aligned} \label{eq:yd}
\end{equation}

Plug Equation (\ref{eq:yy}) and (\ref{eq:yd}) into (\ref{eq:r2}), we have:

\begin{equation}
\begin{aligned}
A^2 &= \frac{\left( r_l s(\|\veta\|^2 - t) + r_l v^2 + v \right)^2}{s\left( r_l^2 s(\|\veta\|^2 - t) + (r_l v + 1)^2 \right)} \\
    &= \frac{s(\|\veta\|^2 - t) + v^2}{s} - \frac{\|\veta\|^2 - t}{r_l^2 s(\|\veta\|^2 - t) + (r_l v + 1)^2} \\
    &= \frac{s(\|\veta\|^2 - t) + v^2}{s} - \frac{1}{\left( \frac{s(\|\veta\|^2 - t) + v^2}{\|\veta\|^2 - t} \right)r_l^2  + \frac{2v}{\|\veta\|^2 - t}r_l + \frac{1}{\|\veta\|^2 - t}}
\end{aligned} \label{eq:final}
\end{equation}

Plug (\ref{eq:final}) into (\ref{eq:testacc}), we then obtain the robust error on all components of the GMM defined in (\ref{eq:gmm}):

\begin{equation}
\centering
\begin{aligned}
\gR(r_l, \epsilon) = \sum_{k = 1}^K p_k \Phi\left( r_k g(r_l) - \epsilon \right),\ g(r_l) = (C_1 - \frac{1}{C_2 r_l^2 + C_3})^{\frac{1}{2}} \\
C_1 = \frac{s(\|\veta\|^2 - t) + v^2}{s},\ C_2 = \frac{s(\|\veta\|^2 - t) + v^2}{\|\veta\|^2 - t},\ C_3 = \frac{2v}{\|\veta\|^2 - t}r_l + \frac{1}{\|\veta\|^2 - t}.
\end{aligned} \label{eq:r_final}
\end{equation}

We study the sign of $C_1$ and $C_2$.
Consider $\rmU = \rmQ\rmQ^T$ is a positive semidefinite matrix, so $s = \vy\rmU^{-1}\vy^T \geq 0$.
In addition, we have $\|\veta\|^2 - t = \veta^T\left(\rmI - (\rmQ\rmQ^T)^{-1} \right)\veta$.
Since $\rmI - (\rmQ\rmQ^T)^{-1} = (\rmI - (\rmQ\rmQ^T)^{-1})^T(\rmI - (\rmQ\rmQ^T)^{-1})$ is a positive semidefinite matrix, we can obtain $\rmI - (\rmQ\rmQ^T)^{-1}$ is also a positive semidefinite matrix.
As a result, $C_1$ and $C_2$ are both non-negative.

\end{proof}

\subsection{Proof of Corollary~\ref{coro:epsilon}} \label{sec:proof_coro}

To prove Corollary~\ref{coro:epsilon}, we first prove the following lemma:

\begin{lemma} \label{lemma:epsilon}
Under the condition of Theorem~\ref{thm:main} and $\gR$ in Equation (\ref{eq:r}), $\frac{\partial \gR(r_l, \epsilon)}{\partial r_l}$ is negative and monotonically decreases with $\epsilon$.
\end{lemma}

\begin{proof}

Based on Equation (\ref{eq:r_final}), we have:

\begin{equation}
\begin{aligned}
\frac{\partial \gR(r_l, \epsilon)}{\partial r_l} = \sum_{k = 1}^K p_k \Phi'(r_k g(r_l) - \epsilon) \frac{\partial g(r_l)}{\partial r_l}
\end{aligned}
\end{equation}

Since the training data is separable, we have $\forall k, r_k \vw^T\veta - \epsilon \|\vw\| > 0$, which is equivalent to the following:

\begin{equation}
\begin{aligned}
\forall k, r_k g(r_l) - \epsilon > 0
\end{aligned}
\end{equation}

First, $p_k$ is a positive number by definition.
Consider function $\Phi(x)$ monotonically decrease with $x$ and is convex when $x > 0$, so $\forall k, \Phi'(r_k g(r_l) - \epsilon)$ is negative and decreases with $\epsilon$.
In addition, $g(r_l)$ increases with $r_l$ and is independent on $\epsilon$, so $\frac{\partial g(r_l)}{\partial r_l}$ can be considered as a positive constant.
Therefore, $\frac{\partial \gR(r_l, \epsilon)}{\partial r_l}$ is negative and monotonically decreases with $\epsilon$.

\end{proof}

Now, we are ready to prove Corollary~\ref{coro:epsilon}:

\begin{proof}

We subtract the left hand side from the right hand side in the inequality of Corollary~\ref{coro:epsilon}:

\begin{equation}
\begin{aligned}
\left[\gR(r_j, \epsilon_1) - \gR(r_i, \epsilon_1)\right] - \left[\gR(r_j, \epsilon_2) - \gR(r_i, \epsilon_2)\right] &= \int_{r_i}^{r_j} \frac{\partial \gR(r_l, \epsilon_1)}{\partial r_l} d_{r_l} - \int_{r_i}^{r_j} \frac{\partial \gR(r_l, \epsilon_2)}{\partial r_l} d_{r_l} \\
&= \int_{r_i}^{r_j} \left[\frac{\partial \gR(r_l, \epsilon_1)}{\partial r_l} - \frac{\partial \gR(r_l, \epsilon_2)}{\partial r_l}\right] d_{r_l} \\
&> 0
\end{aligned} \label{eq:substract}
\end{equation}

The last inequality is based on the conditions $r_j > r_i$, $\epsilon_2 > \epsilon_1$ as well as Lemma~\ref{lemma:epsilon}, they jointly indicate $\left[\frac{\partial \gR(r_l, \epsilon_1)}{\partial r_l} - \frac{\partial \gR(r_l, \epsilon_2)}{\partial r_l}\right]$ is always positive.
We reorganize (\ref{eq:substract}) and obtain $\gR(r_i, \epsilon_1) - \gR(r_j, \epsilon_1) < \gR(r_i, \epsilon_2) - \gR(r_j, \epsilon_2)$.

\end{proof}

\subsection{Proof of Theorem~\ref{thm:nonlinear}} \label{sec:nonlinear_proof}

We start with the following lemma.
\begin{lemma} \label{lemma:nonlinear_1}
Given the assumptions of Theorem~\ref{thm:nonlinear}, we define $g(\vx) = \E(y|\vx)$, $z(\vx) = y - g(\vx)$ and consider $\gamma = \sigma^2_l + h^2(C, \epsilon) - C$, then the following inequality holds.
\begin{equation}
\begin{aligned}
&\forall a \in (0, 1), \sP(\exists f_\vw \in \gF: \frac{1}{n} \sum_{i = 1}^n (y_i - f_\vw(\vx'_i))^2 \leq C) \\ & \leq 2 e^{-\frac{na^2\gamma^2}{8}} + \sP(\exists f_\vw \in \gF: \frac{1}{n}\sum_{i = 1}^n f_\vw(\vx_i)z(\vx_i) \geq \frac{1}{2}(1 - 3a)\gamma)
\end{aligned} \label{eq:nonlinear1}
\end{equation}
\end{lemma}

\begin{proof}
Given the definition of $h(C, \epsilon)$, we have:
\begin{equation}
\begin{aligned}
(y_i - f_\vw(\vx'_i))^2 &= [(y_i - f_\vw(\vx_i)) + (f_\vw(\vx_i) - f_\vw(\vx'_i))]^2 \\
&\geq (y_i - f_\vw(\vx_i))^2 + (f_\vw(\vx_i) - f_\vw(\vx'_i))^2 \\
&\geq (y_i - f_\vw(\vx_i))^2 + h^2(C, \epsilon)
\end{aligned}
\end{equation}

For the first inequality, $\vx'_i$ is the adversarial example which tries to maximize the loss objective, $y_i \in \{-1, +1\}$ and the range of $f_\vw$ is $[-1, +1]$, so $\langle y_i - f_\vw(\vx_i), f_\vw(\vx_i) - f_\vw(\vx'_i) \rangle \geq 0$.
The second inequality is based on the definition of $h^2(C, \epsilon)$ in Definition~\ref{def:h}.
As a result, we can simplify the left hand side of (\ref{eq:nonlinear1}) as follows:

\begin{equation}
\begin{aligned}
\sP(\exists f_\vw \in \gF: \frac{1}{n} \sum_{i = 1}^n (y_i - f_\vw(\vx'_i))^2 \leq C) &\leq \sP(\exists f_\vw \in \gF: \frac{1}{n} \sum_{i = 1}^n (y_i - f_\vw(\vx_i))^2 \leq C - h^2(C, \epsilon))
\end{aligned} \label{eq:remove_adv}
\end{equation}

We consider the sequence $\{z(\vx_i)\}_{i = 1}^n$, it is i.i.d with $\E_{\mu_l}(z(\vx)^2) = \E_{\mu_l}[Var(y|\vx)] = \sigma_l^2$.
Since the range of the prediction is $[-1, +1]$, so $z^2(\vx) \in [0, 4]$.
Then, we have the following inequality by Hoeffding's inequality~\cite{hoeffding1994probability}.

\begin{equation}
\begin{aligned}
\forall a \in (0, 1), \sP(\frac{1}{n} \sum_{i = 1}^n z^2(\vx_i) \leq \sigma^2_l - a\gamma) \leq e^{- \frac{na^2\gamma^2}{8}}
\end{aligned} \label{eq:hoeffding1}
\end{equation}

Similarly, we consider the sequence $\{z(\vx_i)g(\vx_i)\}_{i = 1}^n$, the following inequality holds based on the Hoeffding's inequality and the fact $\E(z(\vx)g(\vx)) = 0$, $z(\vx)g(\vx) \in [-2, +2]$.

\begin{equation}
\begin{aligned}
\forall a \in (0, 1), \sP(\frac{1}{n} \sum_{i = 1}^n z(\vx_i)g(\vx_i) \leq a\gamma) \leq e^{- \frac{na^2\gamma^2}{8}}
\end{aligned} \label{eq:hoeffding2}
\end{equation}

Now we study the right hand side of (\ref{eq:remove_adv}):

\begin{equation}
\begin{aligned}
\frac{1}{n} \sum_{i = 1}^n (y_i - f_\vw(\vx_i))^2 &= \frac{1}{n} \sum_{i = 1}^n \left( z^2(\vx_i) + (g(\vx_i) - f_\vw(\vx_i))^2 + 2z(\vx_i)(g(\vx_i) - f_\vw(\vx_i)) \right) \\
&\geq \frac{1}{n} \sum_{i = 1}^n \left( z^2(\vx_i) + 2z(\vx_i)g(\vx_i) - 2z(\vx_i)f_\vw(\vx_i) \right)
\end{aligned}
\end{equation}

Consider the following reasoning:

\begin{equation}
\begin{aligned}
\left\{
\begin{aligned}
&\frac{1}{n} \sum_{i = 1}^n (y_i - f_\vw(\vx_i))^2 \leq C - h^2(C, \epsilon) = \sigma^2_l - \gamma \\
& \frac{1}{n} \sum_{i = 1}^n z^2(\vx_i) \geq \sigma^2_l - a \gamma \\
&\frac{1}{n} \sum_{i = 1}^n z(\vx_i)g(\vx_i) \geq -a \gamma
\end{aligned}
\right.\ \Longrightarrow \frac{1}{n} \sum_{i = 1}^n z(\vx_i)f_\vw(\vx_i) \geq \frac{1}{2}(1 - 3a)\gamma
\end{aligned} \label{eq:reason}
\end{equation}

As a result, we have:

\begin{equation}
\begin{aligned}
&\sP(\exists f_\vw \in \gF: \frac{1}{n} \sum_{i = 1}^n (y_i - f_\vw(\vx_i) \leq C - h^2(C, \epsilon))) \\
\leq & \sP(\exists f_\vw \in \gF: \frac{1}{n} \sum_{i = 1}^n z^2(\vx_i) \leq \sigma^2_l - a \gamma) 
 + \sP(\exists f_\vw \in \gF: \frac{1}{n} \sum_{i = 1}^n z(\vx_i)g(\vx_i) \geq -a \gamma) +\\
 &\sP(\exists f_\vw \in \gF: \frac{1}{n} \sum_{i = 1}^n z(\vx_i)f_\vw(\vx_i) \geq \frac{1}{2}(1 - 3a)\gamma) \\
 \leq & 2 e^{-\frac{na^2\gamma^2}{8}} + \sP(\exists f_\vw \in \gF: \frac{1}{n} \sum_{i = 1}^n z(\vx_i)f_\vw(\vx_i) \geq \frac{1}{2}(1 - 3a)\gamma)
\end{aligned} \label{eq:reason_bound}
\end{equation}

The first inequality is based on the reasoning of (\ref{eq:reason}).
The second inequality is based on (\ref{eq:hoeffding1}) and (\ref{eq:hoeffding2}).

Based on the inequality (\ref{eq:remove_adv}) and (\ref{eq:reason_bound}), we conclude the proof.

\end{proof}

To further simplify the right hand side of (\ref{eq:nonlinear1}), $\sP(\exists f_\vw \in \gF: \frac{1}{n} \sum_{i = 1}^n z(\vx_i)f_\vw(\vx_i) \geq \frac{1}{2}(1 - 3a)\gamma)$ needs to be bounded, and this is solved by the following lemma.

\begin{lemma} \label{lemma:nonlinear2}
Given the assumptions of Theorem~\ref{thm:nonlinear} and the definition of $g(\vx)$, $z(\vx)$ in Lemma~\ref{lemma:nonlinear_1}, then the following inequality holds.
\begin{equation}
\begin{aligned}
\forall a \in (0, 1), a_1 > 0, a_2 > 0\ \mathrm{and}\ a_1 + a_2 &= \frac{1}{2}(1 - 3a),\\
\sP(\exists f_\vw \in \gF: \frac{1}{n} \sum_{i = 1}^n z(\vx_i)f_\vw(\vx_i) \geq \frac{1}{2}(1 - 3a)\gamma) &\leq 2 |\gF| e^{- \frac{nm}{144cL^2}a^2_1 \gamma^2} + 2 e^{-\frac{n}{8}a^2_2\gamma^2}
\end{aligned} \label{eq:nonlinear2}
\end{equation}
\end{lemma}
\begin{proof}
We recall that the data points $\{\vx_i, y_i\}_{i = 1}^n$ are sampled from the distribution $\mu_l$, which is $c$-isoperimetric.
For any $L$-Lipschitz function $f$, we have:
\begin{equation}
\begin{aligned}
\forall t, \sP[|f_\vw(\vx) - \E_{\mu_l}(f_\vw)| \geq t] \leq 2 e^{-\frac{mt^2}{2cL^2}}
\end{aligned}
\end{equation}

Since $z(\vx) = y - g(\vx) \in [-2, +2]$, we can then bound $\sP[z(\vx)(f_\vw(\vx) - \E_{\mu_l}(f_\vw)) \geq t]$:
\begin{equation}
\begin{aligned}
\forall t, \sP[z(\vx)(f_\vw(\vx) - \E_{\mu_l}(f_\vw)) \geq t] &\leq \sP[|z(\vx)(f_\vw(\vx) - \E_{\mu_l}(f_\vw))| \geq t] \\
&\leq \sP[|(f_\vw(\vx) - \E_{\mu_l}(f_\vw))| \geq \frac{t}{2}] \leq 2 e^{-\frac{mt^2}{8cL^2}}
\end{aligned}
\end{equation}

Here we utilize the proposition in~\cite{vershynin2018high, van2014probability}\footnote{Proposition 2.6.1 in \cite{vershynin2018high} and Exercise 3.1 in \cite{van2014probability}}, which claims \textit{if $\{X_i\}_{i = 1}^n$ are independent variables and all $C$-subgaussian, then $\frac{1}{\sqrt{n}}\sum_{i = 1}^n X_i$ is $18C$-subgaussian.}
Therefore, we have:

\begin{equation}
\begin{aligned}
\forall t, \sP[\frac{1}{\sqrt{n}} \sum_{i = 1}^n z(\vx_i)(f_\vw(\vx_i) - \E_{\mu_l}(f_\vw)) \geq t] \leq 2 e^{-\frac{mt^2}{144cL^2}}
\end{aligned}
\end{equation}

Let $t = a_1\gamma\sqrt{n}$, then we have:

\begin{equation}
\begin{aligned}
\sP[\frac{1}{n} \sum_{i = 1}^n z(\vx_i)(f_\vw(\vx_i) - \E_{\mu_l}(f)) \geq a_1\gamma] \leq 2 e^{-\frac{nm}{144cL^2}a^2_1\gamma^2}
\end{aligned} \label{eq:p1_nonlinear2}
\end{equation}

In addition, we can bound $\sP[\frac{1}{n} \sum_{i = 1}^n z(\vx_i)\E_{\mu_l}(f_\vw) \geq a_2\gamma]$ by:

\begin{equation}
\begin{aligned}
\sP[\exists f_\vw \in \gF: \frac{1}{n} \sum_{i = 1}^n z(\vx_i)\E_{\mu_l}(f_\vw) \geq a_2\gamma] \leq \sP[\frac{1}{n} \sum_{i = 1}^n |z(\vx_i)| \geq a_2\gamma] \leq 2 e^{-\frac{n}{8}a^2_2\gamma^2}
\end{aligned} \label{eq:p2_nonlinear2}
\end{equation}

The first inequality is based on the fact $\E_{\mu_l}(f_\vw) \in [-1, +1]$; the second inequality is based on Hoeffding's inequality.

Now, we are ready to bound the probability $\sP(\exists f_\vw \in \gF: \frac{1}{n} \sum_{i = 1}^n z(\vx_i)f_\vw(\vx_i) \geq \frac{1}{2}(1 - 3a)\gamma)$.

\begin{equation}
\begin{aligned}
&\sP(\exists f_\vw \in \gF: \frac{1}{n} \sum_{i = 1}^n z(\vx_i)f_\vw(\vx_i) \geq \frac{1}{2}(1 - 3a)\gamma) \\
\leq &\sP[\exists f_\vw \in \gF: \frac{1}{n} \sum_{i = 1}^n z(\vx_i)(f_\vw(\vx_i) - \E_{\mu_l}(f)) \geq a_1\gamma] + \sP[\exists f_\vw \in \gF: \frac{1}{n} \sum_{i = 1}^n z(\vx_i)\E_{\mu_l}(f_\vw) \geq a_2\gamma] \\
\leq &2 |\gF| e^{-\frac{nm}{144cL^2}a^2_1\gamma^2} + 2 e^{-\frac{n}{8}a^2_2\gamma^2}
\end{aligned}
\end{equation}
The first inequality is based on the fact $a_1 + a_2 = \frac{1}{2}(1 - 3a)$; the second inequality is based on the Boole's inequality~\cite{boole1847mathematical}, inequality (\ref{eq:p1_nonlinear2}) and (\ref{eq:p2_nonlinear2}).

\end{proof}

To simplify the constant notation, we let $a = \frac{1}{8}$, $a_1 = \frac{3}{16}$ and $a_2 = \frac{1}{8}$.
We plug this into the inequality (\ref{eq:nonlinear1}) and (\ref{eq:nonlinear2}), then:

\begin{equation}
\begin{aligned}
\sP(\exists f_\vw \in \gF: \frac{1}{n} \sum_{i = 1}^n (y_i - f_\vw(\vx'_i))^2 \leq C) \leq 4 e^{-\frac{n\gamma^2}{2^9}} + 2|\gF| e^{-\frac{nm\gamma^2}{2^{12}cL^2}}
\end{aligned} \label{eq:lemma_summary}
\end{equation}

Now we turn to the proof of Theorem~\ref{thm:nonlinear}.

\begin{proof}

We let $\gF_L = \{f_\vw| \vw \in \gW, Lip(f_\vw) \leq L\}$, $\gF_\gamma = \{f_\vw | \vw \in \gW, \vw = \frac{\gamma}{4 J} \odot \vz, \vz \in \sZ^b\}$ and $\gF_{\gamma, L} = \gF_\gamma \cap \gF_L$.
Correspondingly, we let $\gW_L = \{\vw| \vw \in \gW, Lip(f_\vw) \leq L\}$, $\gW_\gamma = \{\vw | \vw \in \gW, \vw = \frac{\gamma}{4 J} \odot \vz, \vz \in \sZ^b\}$ and $\gW_{\gamma, L} = \gW_\gamma \cap \gW_L$.
Because the diameter of $\gW$ is $W$, we have $|\gF_{\gamma, L}| \leq |\gF_\gamma| \leq \left(\frac{4WJ}{\gamma}\right)^b$.
Here, $\odot$ means the element-wise multiplication.

Note that the inequality (\ref{eq:lemma_summary}) is valid for any values of $C$ as long as it satisfies $\gamma \geq 0$.
Based on this, we apply the substitution $\left\{\begin{aligned}C &\leftarrow C + \frac{1}{2}\gamma \\ \gamma &\leftarrow \frac{1}{2}\gamma \end{aligned}\right.$, then:

\begin{equation}
\begin{aligned}
\sP(\exists f_\vw \in \gF_{\gamma, L}: \frac{1}{n} \sum_{i = 1}^n (y_i - f_\vw(\vx'_i))^2 \leq C + \frac{1}{2}\gamma) &\leq 4 e^{-\frac{n\gamma^2}{2^{11}}} + 2|\gF| e^{-\frac{nm\gamma^2}{2^{14}cL^2}} \\
&\leq 4 e^{-\frac{n\gamma^2}{2^{11}}} + 2 e^{b\log(\frac{4WJ}{\gamma})-\frac{nm\gamma^2}{2^{14}cL^2}}
\end{aligned} 
\end{equation}

Based on the definition of $\gW_{\gamma, L}$, we can conclude that $\forall \vw_1 \in \gW_L, \exists \vw_2 \in \gW_{\gamma, L}\ s.t. \|\vw_1 - \vw_2\|_\infty \leq \frac{\gamma}{8J}$.
Therefore, $\forall f_{\vw_1} \in \gF_L, \exists f_{\vw_2} \in \gF_{\gamma, L} s.t. \|f_{\vw_1} - f_{\vw_2}\|_\infty \leq \frac{\gamma}{8}$.
Let choose such $f_{\vw_2} \in \gF_{\gamma, L}$ given an arbitrary $f_{\vw_1} \in \gF_L$, then:

\begin{equation}
\begin{aligned}
(y - f_{\vw_1}(\vx))^2 &= (y - f_{\vw_2}(\vx))^2 + (2y - f_{\vw_1}(\vx) - f_{\vw_2}(\vx))(f_{\vw_2}(\vx) - f_{\vw_1}(\vx)) \\
&\geq (y - f_{\vw_2}(\vx))^2 - \frac{\gamma}{8}|(2y - f_{\vw_1}(\vx) - f_{\vw_2}(\vx))| \\
&\geq (y - f_{\vw_2}(\vx))^2 - \frac{\gamma}{2}
\end{aligned} \label{eq:bound_net}
\end{equation}

The first inequality in (\ref{eq:bound_net}) is based on H\"older's inequality; the second inequality is based on $y \in \{-1, +1\}$ and the range of $\forall f_\vw \in \gF$ is $[-1, +1]$.

We combine (\ref{eq:lemma_summary}) with (\ref{eq:bound_net}), then:

\begin{equation}
\begin{aligned}
\sP(\exists f_\vw \in \gF_{L}: \frac{1}{n} \sum_{i = 1}^n (y_i - f_\vw(\vx'_i))^2 \leq C) &\leq \sP(\exists f_\vw \in \gF_{\gamma, L}: \frac{1}{n} \sum_{i = 1}^n (y_i - f_\vw(\vx'_i))^2 \leq C + \frac{1}{2}\gamma) \\
&\leq 4 e^{-\frac{n\gamma^2}{2^{11}}} + 2 e^{b\log(\frac{4WJ}{\gamma})-\frac{nm\gamma^2}{2^{14}cL^2}}
\end{aligned} \label{eq:conclude}
\end{equation}

Note that $\gF_{L}$ is the set of functions in $\gF$ whose Lipschitz constant is no larger than $L$.
We set the right hand side of (\ref{eq:conclude}) to be $\delta$ and then get $L = \frac{\gamma}{2^7}\sqrt{\frac{nm}{c\left(b\log(4WJ\gamma^{-1}) - \log(\delta/2 - 2e^{-2^{-11}n\gamma^2})\right)}}$.
This concludes the proof.

\end{proof}

\edit{
\subsection{Proof of Corollary \ref{coro:wholeset}} \label{subsec:wholeset}

Based on the definition of $\{\gamma_i\}_{i = 1}^K$, we can apply Theorem~\ref{thm:nonlinear} to each subset of the training set.
Each of these subsets is sampled from one component of the data distribution.
For instances sampled from the $i$-th components, we can derive the lower bound of the model's Lipschitz by the following formulation:
\begin{equation}
\begin{aligned}
Lip^{(i)}(f_\vw) \geq \begin{cases}
0, & \gamma_i < 0\;, \\
\frac{\gamma}{2^7}\sqrt{\frac{nm}{c\left(b\log(4WJ\gamma^{-1}) - \log(\delta/2 - 2e^{-2^{-11}n\gamma^2})\right)}}, & \gamma_i \geq 0\;,
\end{cases}
\end{aligned}
\end{equation}

Since $\{Lip^{(i)}(f_\vw)\}_{i = 1}^K$ are all valid Lipschitz lower bounds for the same model, we can refine the Lipschitz lower bound by choosing the biggest number of them.
We can then get the Lipschitz lower bound as in (\ref{eq:wholeset}).
}

\section{Experimental Settings} \label{sec:app_exp_settings}

\subsection{General Settings} \label{sec:app_exp_settings_general}

The ResNet-18 (RN18) architecture is same as the one in~\cite{wong2020fast}; the WideResNet-34 (WRN34) architecture is same as the one in~\cite{madry2017towards}.
Unless specified, the $l_\infty$ adversarial budget used for CIFAR10 dataset~\cite{krizhevsky2009learning}~\footnote{Data available for download on \href{https://www.cs.toronto.edu/~kriz/cifar.html}{https://www.cs.toronto.edu/~kriz/cifar.html}. MIT license. Free to use.} is $8 / 255$ and for SVHN dataset~\cite{netzer2011reading} \footnote{Data available for download on \href{http://ufldl.stanford.edu/housenumbers/}{http://ufldl.stanford.edu/housenumbers/}. Free for non-commercial use.} is $0.02$.
In PGD adversarial training, the step size is $2 / 255$ for CIFAR10 and $0.005$ for SVHN; PGD is run for $10$ iterations for both datasets.
For adversarial attacks using a different adversarial budget, the step size is always $1/4$ of the adversarial budget's size, and we always run it for $10$ iterations.
To comprehensively and reliably evaluate the robustness of the model, we use AutoAttack~\cite{croce2020reliable}, which is an ensemble of $4$ different attacks: AutoPGD on cross entropy, AutoPGD on difference of logits ratio, fast adaptive boundary (FAB) attack~\cite{croce2020minimally} and square attack~\cite{andriushchenko2020square}.
Unless specified, we use stochastic gradient descent (SGD) with a momentum to optimize the model parameters, we also use weight decay whose factor is $0.0005$.
Unless specified, the momentum factor is $0.9$, the learning rate starts with $0.1$ and is divided by $10$ in the $1 / 2$ and $3 / 4$ of the whole training duration.
The size of the mini-batch is always $128$.

We run the experiments on a machine with 4 NVIDIA TITAN XP GPUs.
It takes about $6$ hours to adversarially train a RN18 model for $200$ epochs, and a whole day to adversarially train a WRN34 model for $200$ epochs.

\subsection{Settings in the Case Studies} \label{subsec:settings_casestudy}

\textbf{Fast Adversarial Training} Our experiments in this section is on CIFAR10 and use the $l_\infty$ norm based adversarial budget with $\epsilon = 8 / 255$.
The step size $\alpha$ in Algorithm~\ref{alg:fast} is $4 / 255$.
Unless explicitly stated, the coefficient $\rho$ and $\beta$ is $0.9$ and $0.1$.
We train the model for $38$ epochs, the learning rate is $0.1$ on the first $30$ epochs, it decays to $0.01$ in the next $6$ epochs and further decays to $0.001$ in the last $2$ epochs.
When we use adaptive targets, the first $5$ epochs are the warmup period in which we use fixed targets.
Since the goal here is to accelerate adversarial training, we do not use a validation set to do model selection as in~\cite{rice2020overfitting}.
We use the standard data augmentation on CIFAR10: random crop and random horizontal flip.

\textbf{Adversarial Fine-tuning with Additional Data} For CIFAR10, we use 500000 images from 80 Million Tiny Images dataset~\cite{torralba200880} with pseudo labels in~\cite{carmon2019unlabeled} \footnote{Data available for download on \href{https://github.com/yguooo/semisup-adv}{https://github.com/yguooo/semisup-adv}. MIT license. Free to use.}.
For SVHN, we use the extra held-out set provided by SVHN itself, which contains 531131 somewhat less difficult samples.
When we construct a mini-batch, half of its instances are sampled from the original training set and the other half are sampled from the additional data.
The experimental settings are the same as~\cite{carmon2019unlabeled} except the learning rate.
We tune the learning rate and find that fixing it to $10^{-3}$ is the best choice.





\section{Additional Experiments and Discussion} \label{sec:app_exp}

\subsection{Properties of the Difficulty Metric} \label{subsec:d_function}

To study the factors affecting the difficulty function defined in~(\ref{eq:difficulty}), let us denote by $d_1$, $d_2$ the difficulty functions obtained under two different training settings, such as different network architectures and training methods.
We then define the difficulty distance (\textit{D-distance}) between two such functions $d_1$, $d_2$ \revision{under the same perturbation type $\gA$} as $D_{\gA}(d_1, d_2)$, which is calculated as follows:
\begin{equation}
\begin{aligned}
D_{\gA}(d_1, d_2) = \E_{\vx \sim U(\gD)} |d_1(\vx, \gA) - d_2(\vx, \gA)|\;.
\end{aligned} \label{eq:d_distance}
\end{equation}
\revision{Similarly, the D-distance between the same function $d$ but under two different perturbation types $\gA_1$, $\gA_2$ is represented by $D_d(\gA_1, \gA_2)$:}
\begin{equation}
\begin{aligned}
D_{d}(\gA_1, \gA_2) = \E_{\vx \sim U(\gD)} |d(\vx, \gA_1) - d(\vx, \gA_2)|\;.
\end{aligned} \label{eq:d_distance}
\end{equation}
\revision{For both $D_{\gA}(d_1, d_2)$ and $D_{d}(\gA_1, \gA_2)$, the expected D-distance between two random difficulty functions with random perturbation types is $0.375$, which is calculated based on the random shuffle of the average loss for each training instance.}

We then study the properties of the difficulty functions in Equation (\ref{eq:difficulty}) by \revision{performing experiments on the CIFAR10 and CIFAR10-C~\citep{hendrycks2018benchmarking} dataset, varying factors of interest and calculating the D-distances between different difficulty functions.}

We first study the influence of the network architectures and training durations by using either a RN18 model, trained for either 100 or 200 epochs (RN18-100 or RN18-200), or a WRN34 model trained for 200 epochs (WRN34).
To generate adversarial attacks, we always use of PGD perturbation \revision{$\gA_{PGD}$} with an adversarial budget based on the $l_\infty$ norm with $\epsilon = 8 / 255$.
This corresponds to the settings used in other works~\cite{hendrycks2018benchmarking, madry2017towards}.
The other hyper-parameters follow the general settings in Appendix~\ref{sec:app_exp_settings}.
In the left part of Table~\ref{tbl:cmp}, we report the D-distance \revision{$D_{\gA_{PGD}}(d_1, d_2)$} for all pairs of settings.
Each result is averaged over $4$ runs, the variances are all below $0.012$ and thus negligible.
The D-distances in all scenarios are very small and close to $0$, indicating the architecture and the training duration have little influence on instance difficulty based on our definition.

\begin{table}[!htb]
\begin{minipage}{.45\linewidth}
\centering
\begin{tabular}{lccc}
\Xhline{4\arrayrulewidth}
$d_1 \backslash d_2$ & RN18-100 & RN18-200 & WRN34 \\
\hline
RN18-100 & $0.0189$ & $0.0232$ & $0.0355$ \\
RN18-200 & $0.0232$ & $0.0159$ & $0.0299$ \\
WRN34    & $0.0355$ & $0.0299$ & $0.0178$ \\
\Xhline{4\arrayrulewidth}
\end{tabular}
\end{minipage}
\hspace{0.05\linewidth}
\begin{minipage}{.45\linewidth}
\centering
\begin{tabular}{lccc}
\Xhline{4\arrayrulewidth}
$\gA_1 \backslash \gA_2$ & Clean & FGSM & PGD \\
\hline
Clean & $0.0189$ & $0.0607$ & $0.1713$ \\
FGSM  & $0.0607$ & $0.0843$ & $0.1677$ \\
PGD   & $0.1713$ & $0.1677$ & $0.0857$ \\
\Xhline{4\arrayrulewidth}
\end{tabular}
\end{minipage}
\caption{D-distances ($D_\gA(d_1, d_2)$ for the left table and $D_d(\gA_1, \gA_2)$ for the right table) between difficulty functions in different settings, including different model architectures, training duration (left table), and different types of perturbations (right table).} \label{tbl:cmp}
\end{table}

We then perform experiments by varying the attack strategy using a RN18 network.
As shown by the D-distances $D_d(\gA_1, \gA_2)$ reported in the right portion of Table~\ref{tbl:cmp}, the discrepancy between values obtained with clean, FGSM-perturbed and PGD-perturbed inputs is much larger, thus indicating that our difficulty function correctly reflects the influence of an attack on an instance.
In addition, Table~\ref{tbl:cmp_common_corr} demonstrates the D-distance between the difficulty functions based on clean instances, FGSM-perturbed instance, PGD-perturbed instances and different common corruptions from CIFAR10-C~\cite{hendrycks2018benchmarking}\footnote{Data available for download on \href{https://github.com/hendrycks/robustness}{https://github.com/hendrycks/robustness}. Apache License 2.0. Free to use.}.
Note that~\cite{hendrycks2018benchmarking} only provides corrupted instances on the test set, so we train models on the clean training set and test model on corrupted test set in these cases.
We use RN18 architecture and train it for $100$ epochs in all cases, results are reported on the test set.
Compared with the results in the left half of Table~\ref{tbl:cmp}, the D-distance is much larger here.
This indicates the difficulty function depends on the perturbation type applied to the input, including the common corruptions.

The results in Table~\ref{tbl:cmp} and~\ref{tbl:cmp_common_corr} demonstrate that our difficulty metric mainly depends on the data and on the perturbation type; not the model architecture or the training duration.
\revision{This is why we include the data $\vx$ and the perturbation type $\gA$ explicitly in the parameter list in the definition of the difficulty function $d$ in Equation~(\ref{eq:difficulty}).}

\begin{table}[!ht]
\centering
\begin{tabular}{p{1.2cm}p{1.7cm}<{\centering}p{1.7cm}<{\centering}p{1.7cm}<{\centering}p{1.7cm}<{\centering}p{1.7cm}<{\centering}p{1.7cm}<{\centering}}
\Xhline{4\arrayrulewidth}
\multirow{2}{*}{$\gA_1 \backslash \gA_2$} & \multirow{2}{*}{brightness} & \multirow{2}{*}{contrast} & \multirow{2}{*}{defocus} & \multirow{2}{*}{elastic} & \multirow{2}{*}{fog} & gaussian \\
& & & & & & blur \\
\hline
Clean & $0.1279$ & $0.3219$ & $0.2646$ & $0.2115$ & $0.2324$ & $0.3069$ \\
FGSM  & $0.1303$ & $0.3128$ & $0.2642$ & $0.2098$ & $0.2289$ & $0.3064$ \\
PGD   & $0.1873$ & $0.3082$ & $0.2616$ & $0.2319$ & $0.2414$ & $0.2959$ \\
\Xhline{4\arrayrulewidth}
\\
\Xhline{4\arrayrulewidth}
\multirow{2}{*}{$\gA_1 \backslash \gA_2$} & glass & \multirow{2}{*}{jpeg} & motion & \multirow{2}{*}{pixelate} & gaussian & impulse \\
& blur & & blur & & noise & noise \\
\hline
Clean & $0.2809$ & $0.1838$ & $0.2520$ & $0.2365$ & $0.2999$ & $0.2869$ \\
FGSM  & $0.2760$ & $0.1853$ & $0.2520$ & $0.2417$ & $0.2918$ & $0.2807$ \\
PGD   & $0.2825$ & $0.2026$ & $0.2605$ & $0.2551$ & $0.2980$ & $0.2866$ \\
\Xhline{4\arrayrulewidth}
\\
\Xhline{4\arrayrulewidth}
\multirow{2}{*}{$\gA_1 \backslash \gA_2$} & \multirow{2}{*}{saturate} & shot & \multirow{2}{*}{snow} & \multirow{2}{*}{spatter} & zoom & speckle \\
& & noise & & & blur & noise \\
\hline
Clean & $0.1335$ & $0.2832$ & $0.2033$ & $0.1930$ & $0.2654$ & $0.2829$ \\
FGSM  & $0.1329$ & $0.2754$ & $0.2003$ & $0.1946$ & $0.2657$ & $0.2759$ \\
PGD   & $0.1932$ & $0.2841$ & $0.2148$ & $0.2297$ & $0.2711$ & $0.2901$ \\
\Xhline{4\arrayrulewidth}
\end{tabular}
\caption{D-distances between difficulty functions of vanilla / FGSM / PGD training and training based on 18 different corruptions on CIFAR10-C. We run each experiment for $4$ times and report the average value.} \label{tbl:cmp_common_corr}
\end{table}

\begin{figure}[!ht]
    \centering
    \includegraphics[width = 0.45\textwidth]{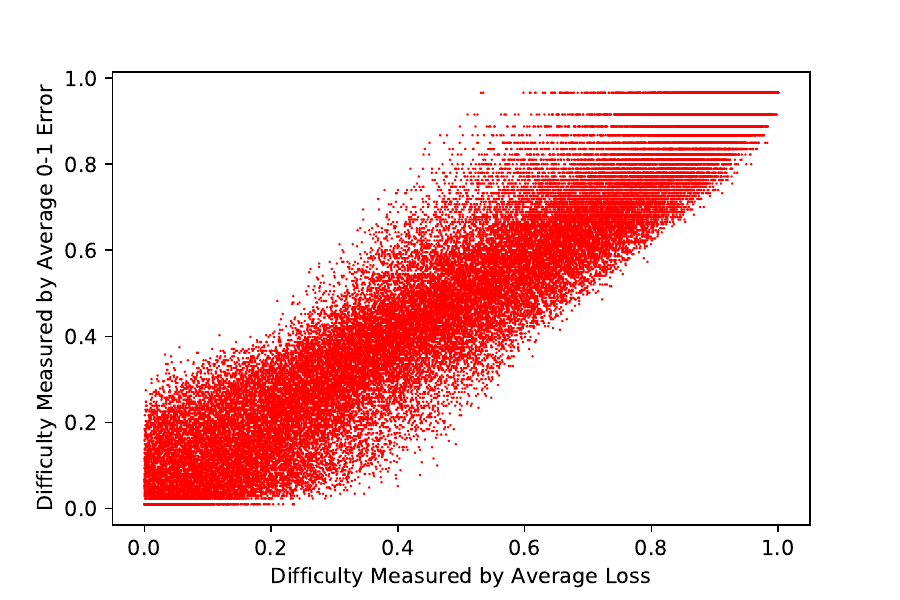}
    ~~~~
    \includegraphics[width = 0.45\textwidth]{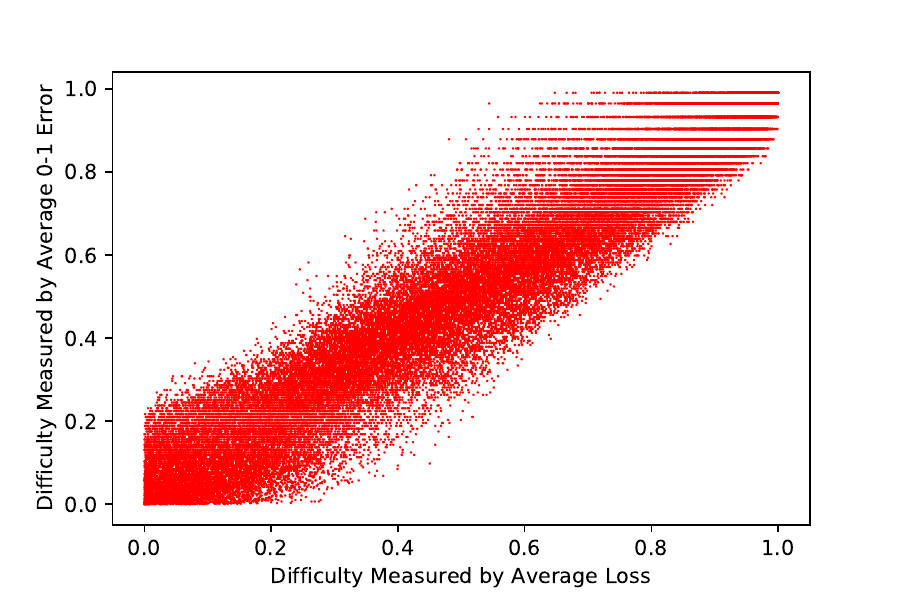}
    \caption{The relationship between the difficulty function based on the average loss values and the one based on the average 0-1 errors. The left figure is based on the RN18-200 model; the right figure is based on the WRN34 model. The correlation between these two metrics are $0.9466$ (left) and $0.9545$ (right), respectively.}
    \label{fig:acc_loss_compare}
\end{figure}

In the definition of our difficulty metric in Equation (\ref{eq:difficulty}), the difficulty of one instance is based on its average loss values during the training procedure.
It is intuitive, because the values of the loss objective represents the cost that model needs to fit the corresponding data point.
The bigger this cost is, the more difficulty this instance will be.
To make the metric stable and prevent the metric from being sensitive to the stochasticity in the training dynamics, we use the average value of the loss objective for each instance to define its difficulty.
In addition to the average loss objectives, we can also use the average 0-1 error to define the difficulty function.
In Figure~\ref{fig:acc_loss_compare}, we plot the relationship between the difficulty metric based on the average loss values and the one based on the average 0-1 error for instances in the CIFAR10 training set when we train a RN18-100 model and a WRN34 model.
We can see a strong correlation between them for both models.
The correlation of the difficulty measured by two metrics for the same instance is $0.9466$ in the RN18-100 case and $0.9545$ in the WRN34 case.
The high correlation indicates we can use either metric to measure the difficulty.
Since the loss objective values are continuous and finer-grained, we choose it as the basis of the difficulty function we use in this paper.

\subsection{Consistency of the Difficulty Definition} \label{subsec:consistent_difficulty}

\edit{The difficulty definitions used in our theoretical analyses and empirical experiments are consistent with the definition of $d$ function in Equation~(\ref{eq:difficulty}) in Section~\ref{sec:hardeasy}.

\textbf{Theoretical Analyses in Section~\ref{sec:thm}} In the analysis of the linear model, we assume the data distribution follows a $K$-component Gaussian mixture model.
In our definition (\ref{eq:gmm}), the average distance between the positive instances and the negative instances of the $k$-th component is $2 r_k$.
Based on symmetry, the average distance between the decision boundary and the adversarial training instances is $r_k + \epsilon$.
Since the loss of the linear model decreases with the increase of the distance between the input and the decision boundary, bigger the value of $r_k$ is, smaller the average loss objective is.
Therefore, in this case, the difficulty level of such training instances, which are defined on their loss objectives, is lower.

In the analysis of the general model, we use the conditional variance $\sigma_k^2$ to represent the difficulty of the $k$-th component of the data distribution.
Based on~\cite{bubeck2021universal}, the conditional variance $\sigma^2_k$ is the average error of a well-trained model.
Since the difficulty is defined on the loss objective, it can be concluded that bigger the $\sigma_k$ is, more difficulty the samples from the corresponding component will be.

\textbf{Case Studies in Section~\ref{sec:casestudy}}

\begin{figure}[!ht]
\centering
\begin{minipage}{.45\linewidth}
\includegraphics[width = 0.95\textwidth]{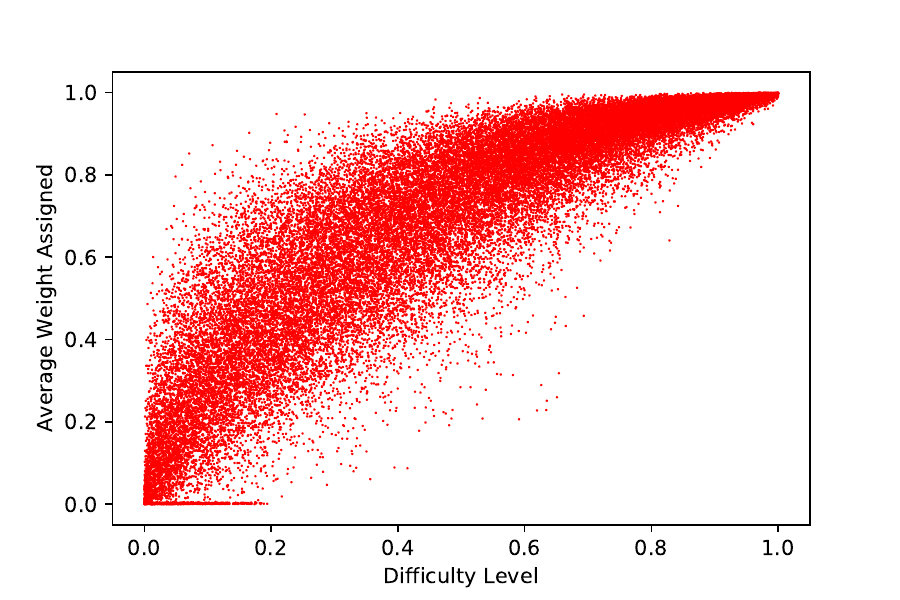}
\caption{The relationship between the difficulty value and the weight assigned to each instances when using reweighting. We use the average weight across epochs. The correlation between them is $0.8900$.} \label{fig:per-rw}
\end{minipage}
\hspace{0.05\linewidth}
\begin{minipage}{.45\linewidth}
\includegraphics[width = 0.95\textwidth]{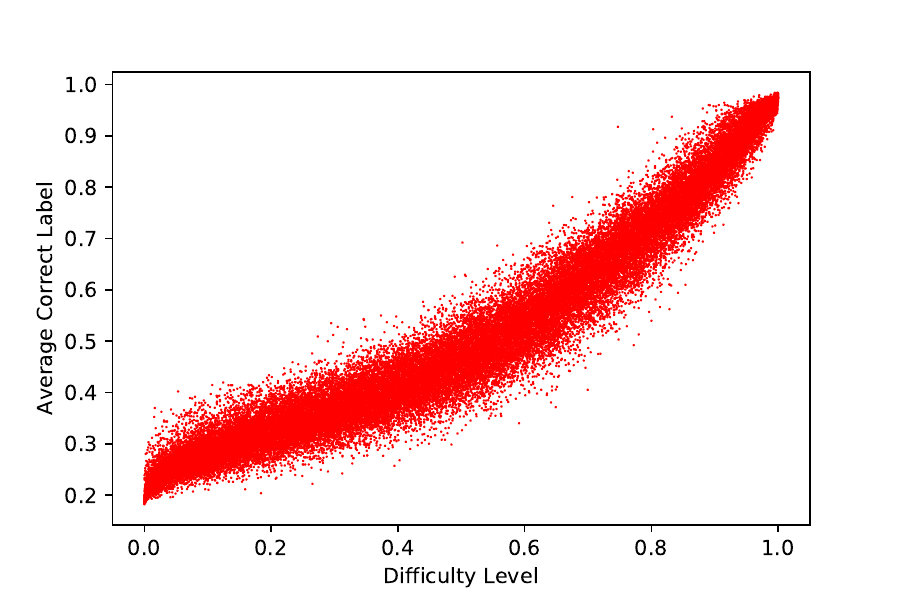}
\caption{The relationship between the difficulty value and the average value of the true label's probability when using the adaptive targets. The correlation between them is $0.9604$.} \label{fig:per-ada}
\end{minipage}
\end{figure}

To confirm that the Algorithm~\ref{alg:fast} is consistent with our difficulty definition, we study the relationship between the instance difficulty and the weight assigned to them when using reweighting, as well as the soft target when using adaptive targets.
Since the evaluation of model robustness is based on the PGD attack, the difficulty value here is also based on the PGD perturbation.
In Figure~\ref{fig:per-rw}, we demonstrate the relationship between the difficulty value and the average assigned weight for each instance when using reweighting.
We calculate the correlation between these two values on the training set, it is $0.8900$.
This indicates we indeed assign smaller weights for hard training instances and assign bigger weights for easy training instances.
In Figure~\ref{fig:per-ada}, we show the relationship between the difficulty value and the average value of the true label's probability in the soft target when we use the adaptive targets.
Similarly, we calculate the correlation between these two values on the training set, it is $0.9604$.
This indicates the adaptive target is similar to the ground-truth one-hot target for the easy training instances, while the adaptive target is very different from the ground-truth one-hot target for the hard training instances.
This means, adaptive targets prevent the model from fitting hard training instances while encourage the model to fit the easy training instances.

}



\subsection{Training on a Subset} \label{sec:app_exp_subset}

\textbf{Results on SVHN dataset}

Figure~\ref{fig:overfit_svhn} demonstrates the learning curves of PGD adversarial training based on a subset of the easiest, the random and the hardest instances of SVHN dataset.
We let the size of each subset be $20000$, because the training set of SVHN is larger than that of CIFAR10.
The model architecture is RN18 in these cases.
We have the same observations here: training on the hardest subset yields trivial performance, training on the random subset has significant generalization decay in the late phase of training while there is no such phenomenon when the model is trained on the easiest instances.

\edit{In Figure~\ref{fig:easy_compare_svhn}, we conduct PGD adversarial training using increasing more training instances in SVHN dataset, starting with the easiest ones.
The observation here is consistent with Figure~\ref{fig:hardremoval}: although fitting hard adversarial instances can cause overfitting, they can improve the model performance if we use easy stopping by a validation set.
Therefore, we should not simply remove the hard training instances, but need to utilize them adaptively.}

\textbf{Different Values of $\epsilon$ and $l_2$-based Adversarial Budget} Figure~\ref{fig:overfit_adv_budget} and Figure~\ref{fig:overfit_l2_adv_budget} demonstrate the learning curves of RN18 models under different adversarial budgets on CIFAR10, in both $l_\infty$ and $l_2$ cases.
In $l_\infty$ cases, the adversarial budgets are $2 / 255$, $4 / 255$ and $6 / 255$; in $l_2$ cases, the adversarial budgets are $0.5$, $0.75$ and $1$.
With the increase in the size of the adversarial budget, we can see a clear transition from the vanilla training: more and more severe generalization decay when training on the random or the hardest subset.


\begin{minipage}{\linewidth}
\centering
\begin{minipage}{0.43\linewidth}
\begin{figure}[H]
\includegraphics[width = \linewidth]{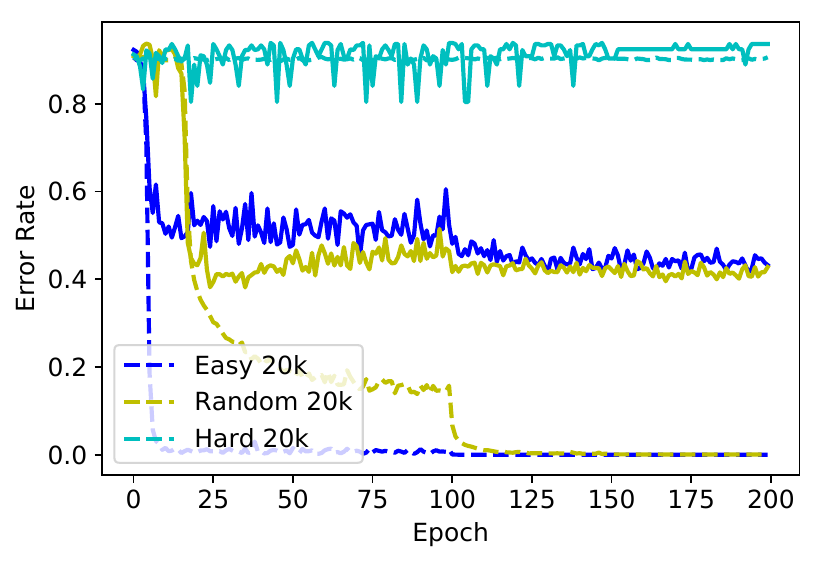}
\caption{Learning curves of training using the easiest, the random and the hardest 20000 instances of the SVHN training set. The training error (dashed lines) is the robust error on the selected instances, and the robust test error (solid lines) is always the error on the entire test set.} \label{fig:overfit_svhn}
\end{figure}
\end{minipage}
\begin{minipage}{0.06\linewidth}
~~
\end{minipage}
\begin{minipage}{0.43\linewidth}
\begin{figure}[H]
\includegraphics[width = \textwidth]{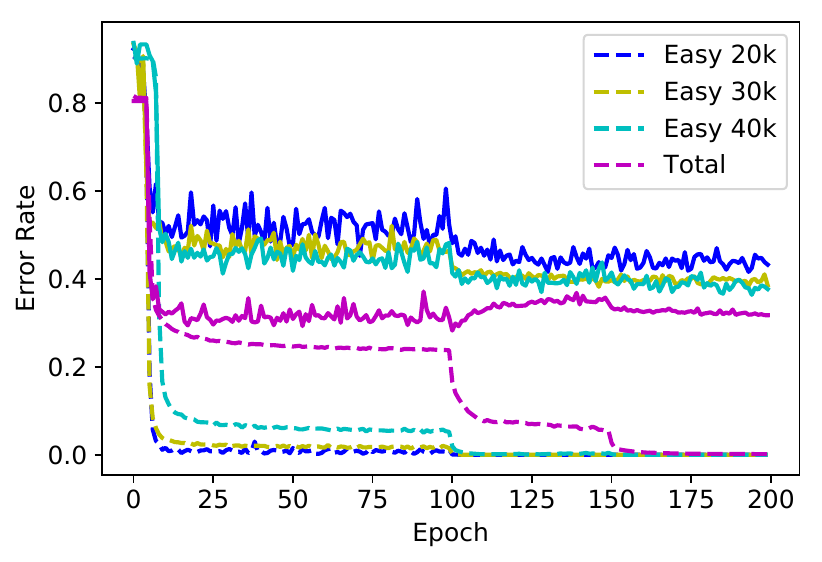}
\caption{Learning curves of PGD adversarial training using increasing more training data of SVHN. The dashed lines represent the robust training error on the selected training instances; the solid lines represent the robust test error on the entire test set.} \label{fig:easy_compare_svhn}
\end{figure}
\end{minipage}
\end{minipage}

\begin{figure}[!ht]
\begin{subfigure}{0.31\columnwidth}
\includegraphics[width = \textwidth]{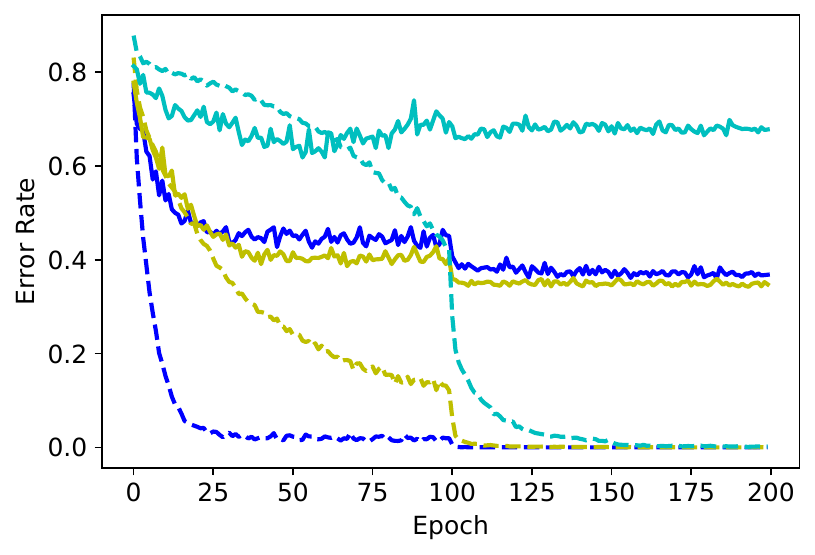}
\caption{$\epsilon = 2 / 255$}
\end{subfigure}
\begin{subfigure}{0.31\columnwidth}
\includegraphics[width = \textwidth]{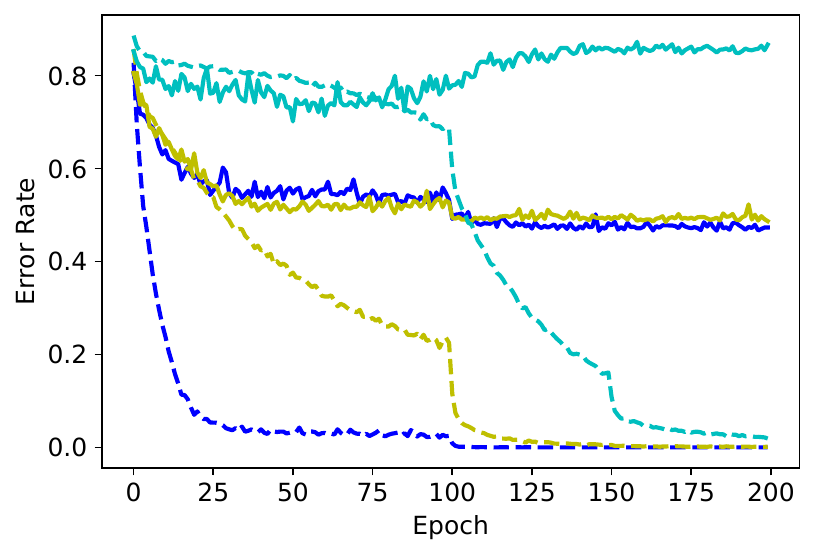}
\caption{$\epsilon = 4 / 255$}
\end{subfigure}
\begin{subfigure}{0.31\columnwidth}
\includegraphics[width = \textwidth]{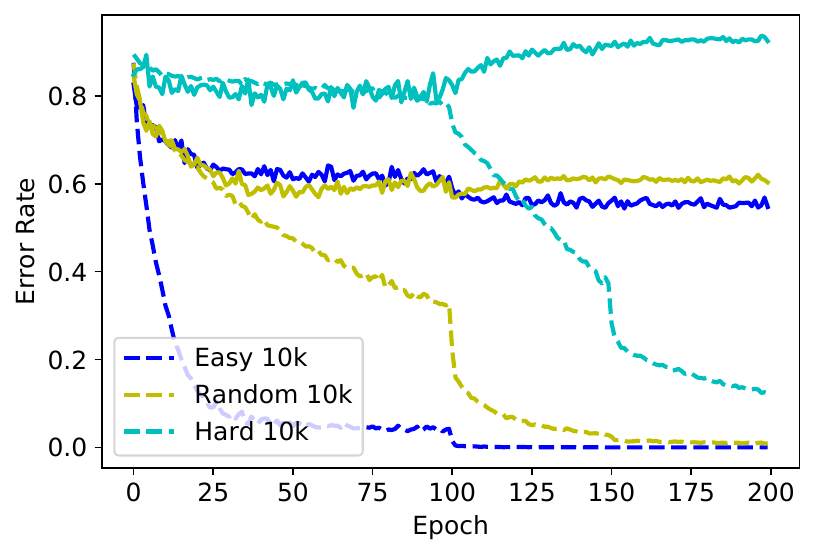}
\caption{$\epsilon = 6 / 255$}
\end{subfigure}
\caption{Learning curves of training on PGD-perturbed inputs against different sizes of $l_\infty$ norm based adversarial budgets using the easiest, the random and the hardest 10000 training instances. The instance difficulty is determined by the corresponding adversarial budget and is thus different under different adversarial budgets. The dashed lines are robust training error on the selected training set, the solid lines are robust test error on the entire test set.} \label{fig:overfit_adv_budget}
\end{figure}

\begin{figure}[!ht]
\begin{subfigure}{.31\columnwidth}
\includegraphics[width = \textwidth]{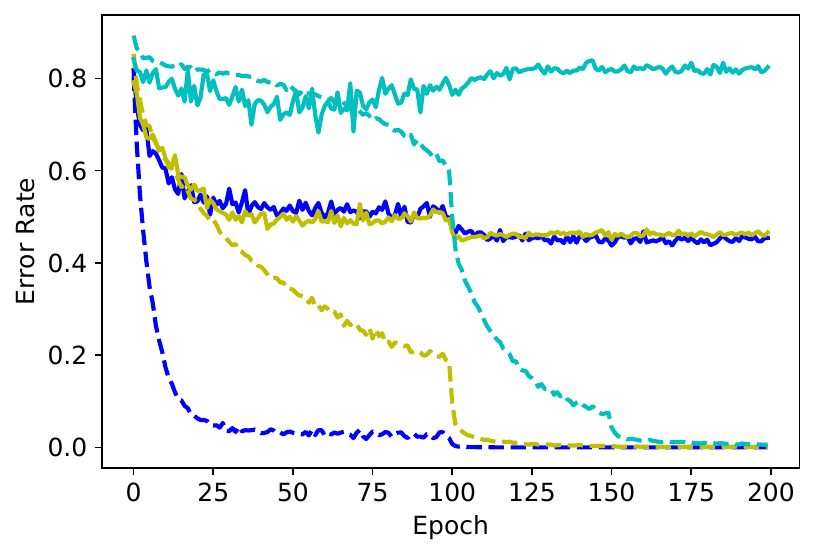}
\caption{$\epsilon = 0.50$}
\end{subfigure}
\begin{subfigure}{.31\columnwidth}
\includegraphics[width = \textwidth]{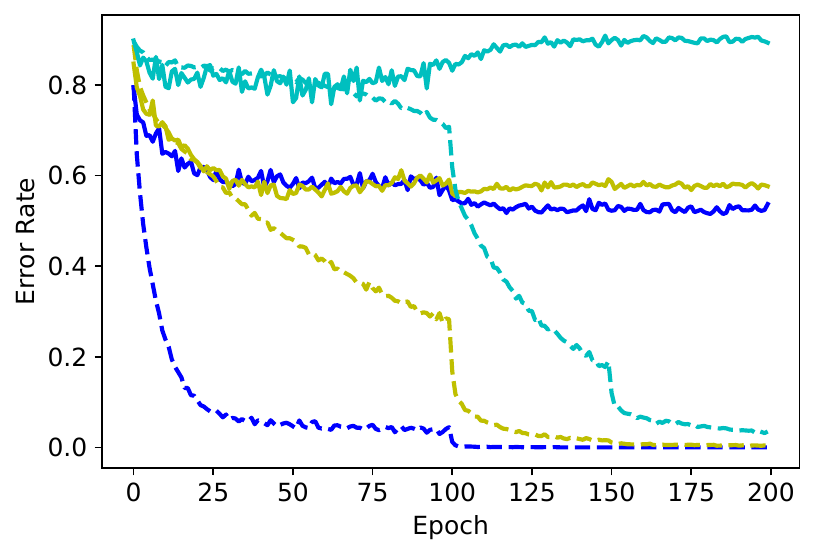}
\caption{$\epsilon = 0.75$}
\end{subfigure}
\begin{subfigure}{.31\columnwidth}
\includegraphics[width = \textwidth]{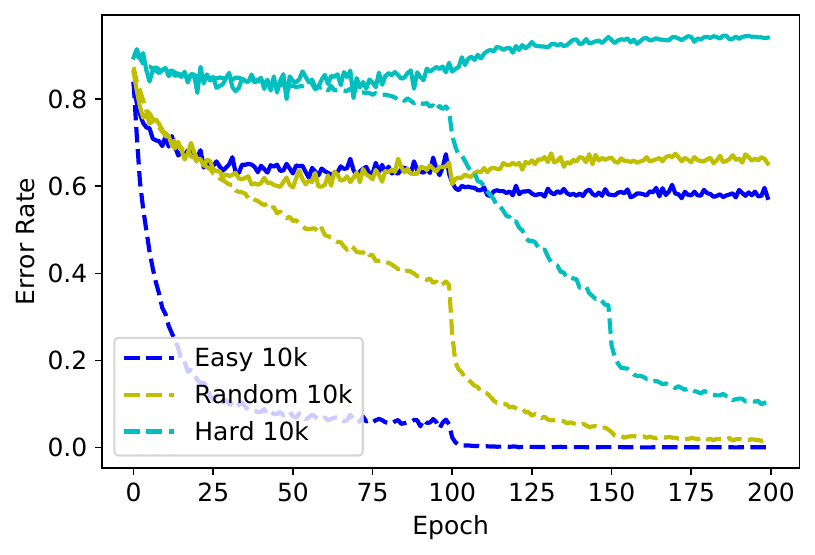}
\caption{$\epsilon = 1.00$}
\end{subfigure}
\caption{Learning curves of training on PGD-perturbed inputs against different size of $l_2$ norm based adversarial budgets using the easiest, the random and the hardest 10000 training instances. The instance difficulty is determined by the corresponding adversarial budget and is thus different under different adversarial budgets. The dashed lines are robust training error on the selected training set, the solid lines are robust test error on the entire test set.} \label{fig:overfit_l2_adv_budget}
\end{figure}


\subsection{Revisiting Existing Methods Mitigating Adversarial Overfitting} \label{subsec:app_revisit}

Existing methods mitigating adversarial overfitting can be generally divided into two categories: one is to use adaptive inputs, such as~\cite{balaji2019instance}; the other is to use adaptive targets, such as~\cite{chen2021robust, huang2020self}.
Both categories aim to prevent the model from fitting hard input-target pairs.
In this section, we pick one example from each category for investigation.
We provide the learning curves of the methods we study in Figure~\ref{fig:case_learncurve}.
We use the same hyper-parameters as in these methods' original paper, except for the training duration and learning rate scheduler, which follow our settings.
These methods clearly mitigate adversarial overfitting: The robust test error does not increase much in the late phase of training, and the generalization gap is much smaller that that of PGD adversarial training.

\begin{figure}
\centering
\includegraphics[width = 0.45\textwidth]{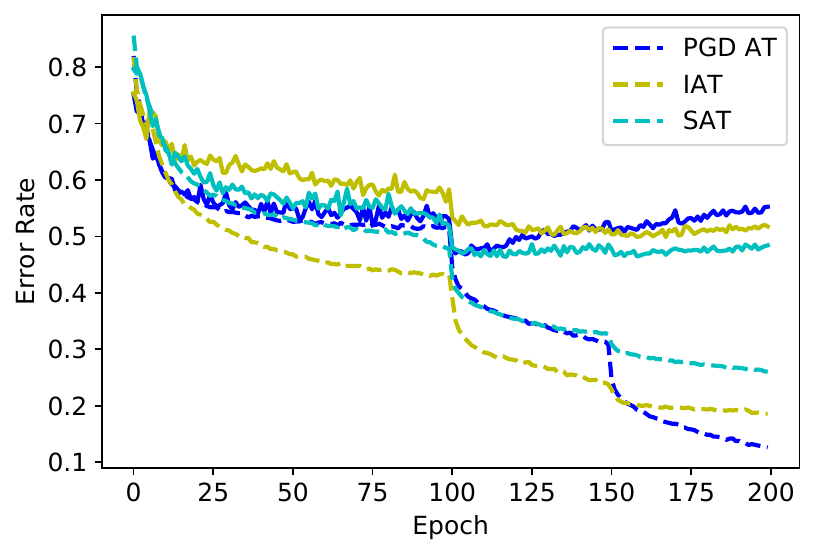}
\caption{Learning curves of PGD adversarial training (PGD AT), instance-adaptive training (IAT) and self-adaptive training (SAT). Dashed lines and solid lines represent the robust training error and the robust test error, respectively.}
\label{fig:case_learncurve}
\end{figure}

\textbf{Instance-Adaptive Training} Using an instance-adaptive adversarial budget has been shown to mitigate adversarial overfitting and yield a better trade-off between the clean and robust accuracy~\cite{balaji2019instance}.
In instance-adaptive adversarial training (IAT), each training instance $\vx_i$ maintains its own adversarial budget's size $\epsilon_i$ during training.
In each epoch, $\epsilon_i$ increases to $\epsilon_i + \epsilon_\Delta$ if the instance is robust under this enlarged adversarial budget.
By contrast, $\epsilon_i$ decreases to $\epsilon_i - \epsilon_\Delta$ if the instance is not robust under the original adversarial budget.
Here, $\epsilon_{\Delta}$ is the step size of the adjustment.

We use the same settings as in~\cite{balaji2019instance} except that we use the same number of training epochs and learning rate scheduling as the one in other experiments for fair comparison.
Specially, we set the value of $\epsilon$ and $\epsilon_{\Delta}$ to be $8 / 255$ and $1.9 / 255$, respectively, same as in~\cite{balaji2019instance}.
The first $5$ epochs are warmup, when we use vanilla adversarial training~\cite{madry2017towards}.



\textbf{Self-Adaptive Training} Self-adaptive training (SAT)~\cite{huang2020self} solves the adversarial overfitting issue by adapting the target.
By contrast with common practice consisting of using a fixed target, usually the ground-truth, SAT adapts the target of each instance to the model's output.
Specifically, after a warm-up period, the target $\vt_i$ for an instance $\vx_i$ is initialized as a one-hot vector by its ground-truth label $y_i$ and updated in an iterative manner after each epoch as $\vt_i \leftarrow \rho \vt_i + (1 - \rho) \vo_i$.
Here, $\rho$ is a predefined momentum factor and $\vo_i$ is the output probability of the current model on the corresponding clean instance.
SAT uses the loss of TRADES~\cite{zhang2019theoretically} but replaces the ground-truth label $y$ with the adaptive target $\vt_i$: $\gL_{SAT}(\vx_i) = \gL(\vx_i, \vt_i) + \lambda \max_{\Delta_i \in \gS(\epsilon)} KL(\vo_i || \vo'_i)$, where $KL$ refers to the Kullback–Leibler divergence and $\lambda$ is the weight for the regularizer.
Furthermore, SAT uses a weighted average to calculate the loss of a mini-batch; the weight assigned to each instance $\vx_i$ is proportional to the maximum element of its target $\vt_i$ but normalized to ensure that all instances' weights sum up to $1$.
By weighted averaging, the instances with confident predictions are strengthened, whereas the ambiguous instances are downplayed.

Similarly, we use the same settings as in~\cite{huang2020self} except we use the same number of training epochs and learning rate scheduling: we train the model for $200$ epochs and the first $90$ epochs are the warmup period.

\vskip 0.2in
\bibliography{main}

\end{document}